%% file: paper.tex
\documentclass{article}

\usepackage{microtype}
\usepackage{graphicx}
\usepackage{subfigure}
\usepackage{hyperref}
\usepackage{booktabs} % for professional tables
\usepackage{enumitem}

\usepackage{hyperref}
\usepackage{thm-restate}

\usepackage[accepted]{icml2024}

\usepackage{amsmath}
\usepackage{amssymb}
\usepackage{mathtools}
\usepackage{amsthm}

\usepackage[capitalize,noabbrev]{cleveref}

\usepackage{hyperref}       % hyperlinks
\usepackage{url}            % simple URL typesetting
\usepackage{booktabs}       % professional-quality tables
\usepackage{amsfonts}       % blackboard math symbols
\usepackage{multirow}
\usepackage{lipsum}
\usepackage{indentfirst}
\usepackage{dsfont}
\usepackage{relsize}
\usepackage{svg}
\usepackage{amsthm}
\usepackage{xspace}
\usepackage{caption}
\usepackage{subcaption}
\usepackage{listings}
\usepackage{threeparttable}
\usepackage{array}
\usepackage{float}
\usepackage{ragged2e}
\usepackage{xcolor}
\usepackage{amssymb}
\usepackage{tabularx}
\usepackage{tikz}
\usepackage{bbm}

\usepackage{comment}
\usepackage{xspace}

\newcommand{\eg}{\textit{e.g.},\xspace}
\newcommand{\ie}{\textit{i.e.},\xspace}
\DeclareMathOperator*{\argmax}{arg\,max}
\DeclareMathOperator*{\argmin}{arg\,min}

\usepackage{lipsum}
\usepackage[export]{adjustbox}

\theoremstyle{plain}
\newtheorem{theorem}{Theorem}[section]

\newtheorem{lemma}[theorem]{Lemma}

\theoremstyle{definition}
\newtheorem{definition}[theorem]{Definition}

\theoremstyle{remark}

\usepackage[textsize=tiny]{todonotes}

 \renewcommand{\paragraph}[1]{
\noindent \textbf{#1.~} 
 }
\setenumerate{itemsep=0pt,leftmargin=0.1in}
\setitemize{itemsep=0pt,leftmargin=0.1in}

\icmltitlerunning{Online Cascade Learning for Efficient Inference over Streams}

\begin{document}

\twocolumn[
\icmltitle{Online Cascade Learning for Efficient Inference over Streams}
% Online Model Cascade for Cost-Efficient Streaming Data Inference



\icmlsetsymbol{equal}{*}

\begin{icmlauthorlist}
\icmlauthor{Lunyiu Nie}{ut}
\icmlauthor{Zhimin Ding}{rice}
\icmlauthor{Erdong Hu}{rice}
\icmlauthor{Christopher Jermaine}{rice}
\icmlauthor{Swarat Chaudhuri}{ut}
\end{icmlauthorlist}

\icmlaffiliation{ut}{The University of Texas at Austin}
\icmlaffiliation{rice}{Rice University}

\icmlcorrespondingauthor{Lunyiu Nie}{\href{mailto:lynie@utexas.edu}{lynie@utexas.edu}}
\icmlcorrespondingauthor{Swarat Chaudhuri}{\href{mailto:swarat@cs.utexas.edu}{swarat@cs.utexas.edu}}

\icmlkeywords{Machine Learning, ICML}

\vskip 0.3in
]



\printAffiliationsAndNotice{}  % leave blank if no need to mention equal contribution

\begin{abstract}
Large Language Models (LLMs) have a natural role in answering complex queries about data streams, but the high computational cost of LLM inference makes them infeasible in many such tasks.
We propose \emph{online cascade learning} as an approach to address this challenge. The objective here is to learn a ``cascade'' of models, starting with lower-capacity models 
(such as logistic regression) and ending with a powerful LLM, along with a \emph{deferral policy} that determines the model to be used on a given input. We formulate the task of learning cascades online as an imitation-learning problem, where smaller models are updated over time imitating the LLM expert demonstrations, and give a no-regret algorithm for the problem. 
Experimental results across four benchmarks show that our method parallels LLMs in accuracy while cutting down inference costs by as much as 90\% with strong robustness against input distribution shifts, underscoring its efficacy and adaptability in stream processing.

% The growing use of Large Language Models (LLMs) in commercial applications has raised concerns about the high computational overheads during inference, particularly in streaming setups where queries arrive continuously.  
% Therefore, we are driven to build a cascade system designed to route simpler queries to less resource-intensive models and more complex queries to higher-end, costlier models. Specifically, we propose an online cascade learning framework that combines online learning and deferral policy learning to optimize both response quality and cost efficiency. Our framework stands out from previous cascade methods by its ability to dynamically learn and adapt in real-time, thus improving its handling of queries over streams. We formulate this as a multi-objective optimization problem and design an imitation learning algorithm based on an episodic Markov decision process, complete with a theoretical no-regret guarantee. Experimental results across various benchmarks show that our method not only parallels LLMs in performance but also cuts down inference costs by as much as 90\%, underscoring its efficacy and adaptability in stream processing.
\end{abstract}

\input{sections/1_intro}

\input{sections/2_background.tex}
\input{sections/3_exp}
\input{sections/4_related}

\input{sections/5_conclu}

\section*{Acknowledgements}
We thank the anonymous reviewers for their valuable comments and suggestions in enhancing the paper. This research was supported by the NSF under grant numbers CCF-1918651, 2008240, 2131294, and 2212557, ARO award \#W911NF-21-1-0009, DARPA award \#HR00112320018, NIH CTSA award No. UL1TR003167, and US DOT Tier-1 UTC CYBER-CARE grant \#69A3552348332.

\section*{Impact Statement}
This work presents an online cascade learning framework for enhancing the efficiency of LLMs in handling streaming queries. While primarily advancing the field of machine learning, it also has broader societal implications. Notably, by significantly reducing computational demands (up to 90\% in inference costs), our framework addresses environmental concerns related to the energy consumption of large-scale computing. This reduction in resource utilization aligns with global efforts towards sustainable technology use.

Furthermore, our work has the potential to democratize access to advanced machine learning technologies. By lowering operational costs, smaller organizations and researchers with limited resources can leverage the power of LLMs, fostering inclusivity and diversity in research and application development. Additionally, our approach could contribute to sectors like healthcare, finance, and public policy, where real-time, cost-effective data processing is crucial. 


\bibliography{ref}
\bibliographystyle{icml2024}

%%%%%%%%%%%%%%%%%%%%%%%%%%%%%%%%%%%%%%%%%%%%%%%%%%%%%%%%%%%%%%%%%%%%%%%%%%%%%%%
%%%%%%%%%%%%%%%%%%%%%%%%%%%%%%%%%%%%%%%%%%%%%%%%%%%%%%%%%%%%%%%%%%%%%%%%%%%%%%%
% APPENDIX
%%%%%%%%%%%%%%%%%%%%%%%%%%%%%%%%%%%%%%%%%%%%%%%%%%%%%%%%%%%%%%%%%%%%%%%%%%%%%%%
%%%%%%%%%%%%%%%%%%%%%%%%%%%%%%%%%%%%%%%%%%%%%%%%%%%%%%%%%%%%%%%%%%%%%%%%%%%%%%%
\newpage
\appendix
\onecolumn
\input{sections/a_appendix}
\input{sections/b_appendix}

\end{document}

%% file: sections/1_intro.tex
\newcommand{\llama}{Llama\xspace}
\newcommand{\approach}{\textsc{Omc}\xspace}

\section{Introduction}

Large language models (LLMs) \citep{bommasani2021opportunities,touvron2023llama,brown2020language} hold great promise as a means of one-pass query answering over text streams. For example, suppose we have a stream of movie reviews posted on the internet and would like to retrieve the ones that are positive (\Cref{fig:exp}). 
To address this without human annotations, we can create a prompt for the LLM to identify the sentiment in a review. The LLM can then use this prompt to process each statement in sequence.

%In such an application, a natural language prompt could specify the documents that are to be retrieved from the stream. 
%The query is ``zero shot’’ in that there are no positive examples available for training. 

%Further, in many applications an unbiased set of positive examples would be impossible to obtain, and the goal of the computation is to obtain the few documents in the stream matching the query predicate. For just one application, consider a real-time stream of messages posted on the internet, and a fact-checking 

%While LLMs can do a remarkable job of finding those documents matching such an intricate specification, 

% However, a fundamental difficulty here is that inference using LLMs is expensive, especially in a streaming setup where queries arrive continuously. 
However, LLM inference can be extremely expensive, especially in a streaming setup where queries arrive continuously. 
% Based on a recent analysis \citep{samsi2023words}, the \llama-65B \citep{touvron2023llama1} model has a latency close to 200 ms per token, and consumes approximately 600 Joules of energy per second, for inference on the GSM8K task \citep{cobbe2021training} on four A100 GPUs. Deploying such a model in a streaming setting can quickly become unsustainable. 
Based on our benchmarking (Appendix~\ref{sec:prefill}), 
a GPU server with eight A100 GPUs takes 3.6 seconds to process a document containing 8,192 tokens using the largest \llama \cite{touvron2023llama1}. To process one million such documents per hour would thus require 1,000 A100 servers. Using Amazon Web Services, this computation would cost more than 30,000 USD per hour, if it were even possible to procure the machines required.

\begin{figure}[t]
     \centering
     \includegraphics[width=0.8\columnwidth]{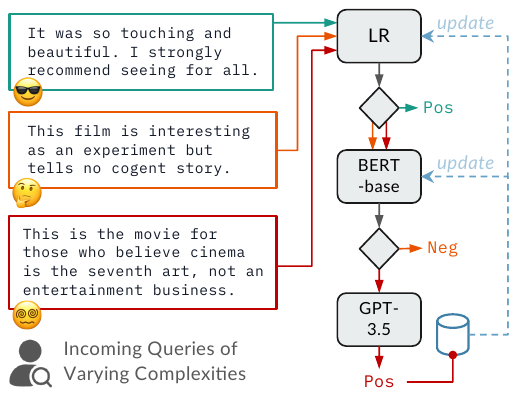}
     \caption{
% A task consisting of a series of fact-checking queries from the FEVER \cite{thorne2018fever} benchmark. The queries have varying complexities. Our cascade approach is based on the observation that the simpler queries can be efficiently processed by the cheapest logistic regression model (green lines), whereas more complex queries are deferred to larger models (orange and red lines).  When the cascade proceeds to the LLM, the annotations are collected to update the smaller models online.
A sentiment analysis task over a stream of IMDB movie reviews  \cite{imdb}. We use the cheapest logistic regression model (green lines) to process simpler queries and defer more complex queries to the larger models (orange \& red lines). When the cascade proceeds to the LLM, the annotations are collected to update the smaller models online (blue lines).
} 
     \label{fig:exp}
     \vspace{-0.2in}
\end{figure}

There are two popular ways to reduce the cost of LLM inference. The first is to 
\emph{distill} a large language model into a smaller model that can process a document using less computation \cite{hinton2015distilling, gu2023knowledge, hsieh2023distilling}. The second is to use a \emph{cascade} of models that use smaller models to process ``easier'' inputs and reserve the largest models for the most difficult inputs~\cite{DBLP:conf/emnlp/VarshneyB22, frugalgpt}.
However, %none of these existing proposals is suitable for use in a stream-based setting. All of them 
these existing proposals assume a model of learning in which a labeled training set is available beforehand, making them unsuitable for a streaming setting. In contrast, the streaming setting requires \emph{online learning} \citep{hoi2021online}.

\begin{figure*}[t!]
     \centering
     \includegraphics[width=.8\textwidth]{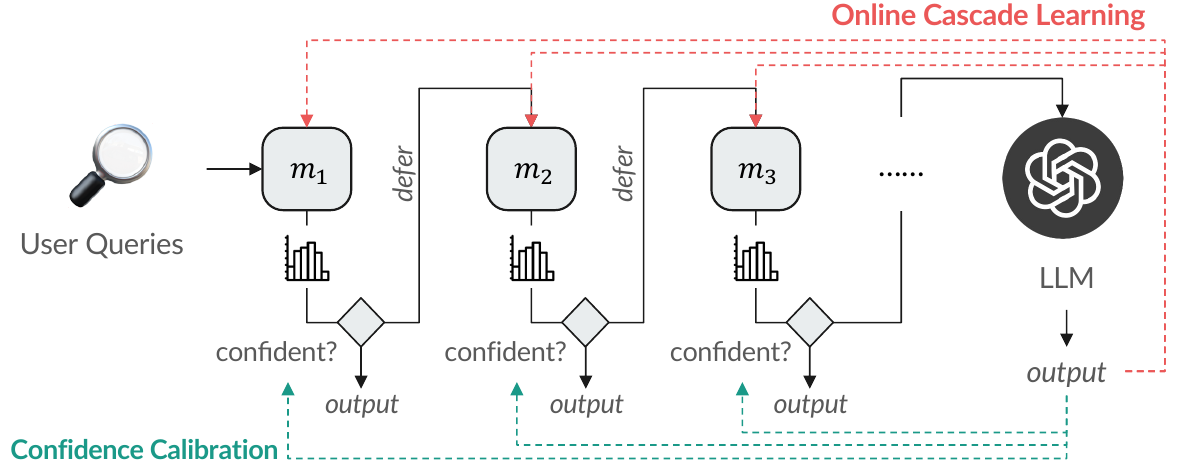}
     \caption{The proposed online cascade learning framework, where smaller models with monotonically increasing capacities and costs ($c_1 < c_2 < ... < c_N$) can progressively learn from the ongoing outputs of an LLM (as denoted in red arrows). Meanwhile, the deferral policy and corresponding confidence scores are also calibrated online (in green arrows). } 
     \label{fig:frame}
\end{figure*}

%This is not realistic if the goal is exploratory processing --- where the whole point is finding the small number of positive examples in a data set—or if the stream is produced in real-time so and so it is impossible to obtain training documents from an unknown, future distribution.
In this paper, we propose \emph{online cascade learning} to bridge this gap in the literature. 
Our goal here is to develop a cascade of models, arrayed from the least to the most complex, that can process queries in a data stream with optimal cost-performance trade-offs. The critical difference from prior work is that our cascades are trained in a fully online manner and do not require any human-labeled training data. The smaller models in the cascade would continuously evolve and improve over time by imitating the LLM's demonstrations on the harder queries, enabling them to handle an expanding range of queries with increasing proficiency.
%, so  accepts its input query and immediately starts processing. 

A key component of our cascades is the \emph{deferral policy} that decides, given an input, the best ``level" of the cascade that should handle the input (\Cref{fig:frame}).
% has an associated, learned  that looks at the model output to determine whether the answer provided by the model should be treated, or whether a more expensive model should be used. 
At startup, the policy keeps its ``gates'' open, allowing all initial inputs to flow through the cascade and be processed by the most expensive model (an LLM). These processed inputs then become training labels for updating the smaller models and the deferral policy within the cascade. Over time, as the model sees more data, the system stabilizes at a state where the smaller, less expensive models can handle the majority of the new inputs. The framework also incorporates a set of learning hyperparameters that adjust the trade-off between accuracy and cost based on user needs.

We formalize the problem of online cascade learning
%as a multi-objective optimization problem --- more precisely, 
in terms of an episodic Markov decision process (MDP) that considers both prediction loss and computational costs for co-optimization. We assume an expert policy --- a high-capacity LLM --- for this MDP. We learn the various components of the cascade through imitation learning \citep{ross2011reduction} based on the LLM demonstrations. We show that our algorithm comes with a theoretical no-regret guarantee. Our experimental results, on four tasks of various complexity, show that our proposed method can achieve accuracy comparable to LLM at a vastly reduced inference cost.

To summarize the main contributions of our work: 
\begin{itemize}[leftmargin=0.15in,itemsep=0em]
    \item We introduce \textit{online cascade learning}, a new framework for learning model cascades in resource-intensive streaming analytics settings. The framework enables systematic trade-offs between prediction accuracy and resource usage, and allows learning without any human annotations. 
    \item We offer a formulation of the online learning of cascades in terms of episodic MDPs and give a no-regret imitation learning method for solving this problem. 
    \item We present rigorous experiments showing that our proposed algorithm can achieve comparable accuracy as LLMs while saving up to $90\%$ of the inference costs. Our source code is available at \url{https://github.com/flitternie/online_cascade_learning}.
\end{itemize}

%% file: sections/2_background.tex
\section{Problem Formulation}
%\subsection{Inference over Streams as an MDP Problem}

% \swarat{I would propose a bit of a reorganization. The functions $m_1,\dots, m_k$ are part of our \emph{solution} to the problem of learning to answer queries. I would first start with the problem, then give the solution. This means: 
% \begin{itemize}
% \item First define the problem (and the MDP) without referencing the functions $m_i$. Note that the MDP depends on the \emph{shape} of the cascade, but the specific functions in the cascade are learned part of the policy
% \item define policies $\pi$ for the MDP as having a branching structure mirroring the cascade. I think we can give an explicit programmatic representation:
% $\pi(\langle x_t, i \rangle) = \mathbf{if}~f_i(x_t, i)~\mathbf{then}~\mathtt{defer}~ \mathbf{else}~m_i(x_t)$ (may not be fully correct but you get the point). Say that the functions $f_i$ are called the deferral policy. 
% \item Now define an optimization problem over policies (which include functions $f_i$ and $m_i$) and say that our goal is to learn the optimal policy. This makes it clear that we are learning both the deferral policy and the models themselves.  
% \item In the next section, say that the entire policy learning problem (learning $\pi$) can be framed as Dagger? Does any of the math break if we do this?
% \end{itemize}
% }

\newcommand{\Init}{\overline{\mathcal{X}}}
\newcommand{\X}{\mathcal{X}}
\newcommand{\Trans}{\mathcal{T}}

\paragraph{Inference over Streams as an MDP Problem}
We consider stream processing scenarios that have as input a fixed infinite stream $X = \langle x_1, \dots, x_t, \dots\rangle$ of user queries.
The $t$-th query $x_t$ is associated with a ground-truth label $y_t \in Y$, where $Y$ is a label set. 
Our goal is to predict the label for each $x_t$ using an $N$-level model cascade. 

We formulate our problem using an episodic Markov decision process (MDP) $(S, A, \Trans, C)$. Here:
\begin{itemize}[leftmargin=0.1in,itemsep=0em]
\itemsep0em 
%\item $N$ is the maximum number of actions in an episode. 
\vspace{-0.2em}
    \item $S$ is a set of states. A state in the $t$-th episode is either a pair $\langle x_t, i\rangle$, where 
    $i \in \{1, ..., N\}$ indicates the current cascade level, or a special 
   terminal state \texttt{exit} that ends the episode. 
   The initial state of the $t$-th episode is $\langle x_t, 1\rangle.$
%     \begin{itemize}[leftmargin=0.1in,itemsep=0em]
%     \itemsep0em 
%         %\item $\langle x_t, y_t \rangle$ are sampled from $\mathcal{X}$ to start the $t$-th episode. 
% %        , sampled IID from a distribution $\X$, at episode $t$.
%         \item $i \in \{1, ..., N\}$ indicates the current cascade level.
%         \item \texttt{exit} is a special ``terminal state'' that ends the episode.
%         %and resets the MDP with $\langle x_{t+1}, 1\rangle$ as the initial state of a new episode.
%     \end{itemize}
    % We use the shorthand $s_{t,i}$ to denote the state $\langle x_t, y_t, i \rangle$.
    For clarity, we abbreviate $\langle x_t, i \rangle$ by $s_{t,i}$.
    
  %  \item $\Init$ is a map specifying the initial The initial state distribution $\Init$ is realized by setting the initial state to $\langle x_t, y_t, 1 \rangle$.    
    \item $A$ is a set of actions, consisting of: 
    \begin{itemize}
    \itemsep0em 
        \item The label set $Y$, representing the potential predictions if the cascade chooses to output at the current state. For instance, in a binary classification task, \( Y=\{0,1\} \). 
        \item A special action \texttt{defer} that activates the next level of the cascade. %Intuitively, this action is used when the model at the current cascade level cannot make a high-confidence prediction. 
    \end{itemize}
    \item $\Trans(s_{t,i}, a)$ is a deterministic transition function, consisting of transitions of the form:  
    \begin{itemize}
    \itemsep0em 
        \item $\Trans(s_{t,i}, a) = \texttt{exit}$ for $a \in Y$.
        %$\langle x_t,i\rangle \xrightarrow[]{a \in Y} \texttt{exit}$
        \item $\Trans(s_{t,i}, \texttt{defer}) = s_{t,i+1}$. 
        %$\langle x_t,i\rangle \xrightarrow[]{a = \texttt{defer}}\langle x_t,i+1 \rangle$
    \end{itemize}

    \item $C(s_{t,i}, a)$ is a cost function defined as:
        \begin{equation*}
        C(s_{t,i},a)=\begin{cases}
    \mathcal{L}(a|y_t)& \text{if $a\in Y$},\\
    \mu  c_{i+1} & \text{if $a=$\texttt{defer}.}
  \end{cases}
        \end{equation*}
    Here, $\mathcal{L}(a|y_t)$ is a \emph{prediction loss} that measures the accuracy of the cascade's prediction. $c_{i+1}$ represents the penalty we pay for a deferral --- intuitively, this penalty captures the computational overheads of going one level deeper into the cascade. 
    The adjustable constant $\mu$ guides the trade-off between computational cost and accuracy.
\end{itemize}
\paragraph{Online Cascade Learning} 
We now formulate online cascade learning as the problem of solving the above episodic MDP.
% , where the maximum length of each episode is $N$, the number of levels in the cascade. Within this MDP, both states and actions are delineated to be independent across different episodes and input queries. 
%The %primary 
%objective in our learning problem 
%of online cascade learning 
%is to minimize the cumulative cost across all conceivable states over a time horizon of $T$ episodes. This cost function should encompass two critical aspects: model prediction loss and computational costs. Thus, minimizing the cumulative cost is inherently a task of achieving a delicate balance between enhancing predictive accuracy and ensuring cost efficiency. 
% Formally, we use $\Pi$ to denote the complete policy space under consideration, which is the theoretical aggregation of all trainable model weights in the cascade. Within this space, for a given policy $\pi$, \ie a cascade of models of certain weights, the notation $d_{\pi}$ is employed to signify the expected state distribution when the policy $\pi$ is applied consistently throughout an entire episode.
Let a \emph{policy} $\pi$ be a stochastic map from states to actions. We use $\pi (s_{t,i}, \texttt{defer})$ to represent the probability that $\pi$ chooses to defer in state $s_{t,i}$, and $\pi (s_{t,i}, y)$ to represent the probability that $\pi$ chooses $y \in Y$, conditioned on no deferral having occurred.
% , where $\pi (s_{t,i}, \texttt{defer})$ is the probability that $\pi$ chooses to defer in state $s_{t,i}$, and $\pi (s_{t,i}, y)$ is the probability that $\pi$ chooses $y \in Y$, conditioned on no deferral. 
% Let $\Pi$ denote a universe of possible policies. 
%complete policy space under consideration.
% which includes the tunable model weights ($m_1, ..., m_{N-1}$), and the deferral functions ($f_1, ..., f_{N-1}$). Within this space, 

The probability of a policy $\pi$ entering state $s_{t,i}$ in episode $t$ is denoted as $p^{s_{t,i}}_{\pi}$. For $i > 1$, $s_{t,i}$ can be reached if and only if the policy chooses the \texttt{defer} action in all preceding states $s_{t,1}, ..., s_{t,i-1}$ within the current episode. Thus, 
$$p^{s_{t,1}}_\pi = 1, \qquad\qquad  p^{s_{t,i}}_{\pi} = \prod^{i-1}_{j=1}\pi (s_{t,j}, \texttt{defer}).$$ 
Then, the %expected 
cost of executing $\pi$ over $T$ episodes is computed by summing over all the episodes and cascade levels: 
\begin{align}
% J(\pi) = \mathbb{E}_{\langle x_t, y_t \rangle \sim \mathcal{X}}\left[ \sum_{i = 1}^N p^{s_{t,i}}_{\pi} C_{\pi}(s_{t,i})\right] 
J(\pi, T) = \sum_{t=1}^{T}\left[ \sum_{i = 1}^N p^{s_{t,i}}_{\pi} C_{\pi}(s_{t,i})\right] 
\label{eq:loss}
\end{align}

Here, $C_\pi(s_{t,i})$ is the expected, immediate cost of applying policy $\pi$ at state $s_{t,i}$.  This is computed as: 
\begin{align*}
    C_\pi(s_{t,i}) &=  \pi (s_{t,i}, \texttt{defer}) \cdot \mu c_{i+1} \\ &+ (1 - \pi (s_{t,i}, \texttt{defer})) \cdot \sum_{y \in Y} \pi (s_{t,i}, y)  \cdot \mathcal{L}(y | y_t). 
\end{align*}

After having seen the first $T$ queries in the input stream $X$, our learning goal is to find a policy that minimizes $J(\pi, T)$.
When $T$ is clear from the context, we often abbreviate $J(\pi, T)$ by $J(\pi)$. 

\paragraph{Policy Representations}
We represent policies in a factorized way using a set of \emph{classification models} 
$\langle m_1,\dots,m_N\rangle$ 
that constitute the different levels of the cascade, and a set of \emph{deferral functions} 
$\langle f_1,\dots, f_{N -1}\rangle$ that decide whether the current level can perform a high-confidence classification or to defer.
%($f_N$ never chooses to defer).
We assume each $m_i$ to produce a vector of probabilities, with one probability for each label, and each $f_i$ to produce a probability of deferral. Then the overall policy has the form:  
\begin{align}
\pi(s_{t,i}, \texttt{defer}) &= f_i (m_i (x_t)), \nonumber \\
\pi(s_{t,i}, y) &=  \left(m_i (x_t) \right)[y] \textrm{ for } y \in Y.  \nonumber
\end{align} 

We assume the two parameterized function representations for each level $i$: the classification model $m_i$ and the deferral function $f_i$, are both characterized by a crucial property: they guarantee uniform computational costs for evaluating the functions, regardless of the specific function parameters and function inputs. This means that for any instantiation of the parameters in $m_i$ and $f_i$, the inference costs remain constant, irrespective of the specific input query. This assumption is reasonable in the practical scenarios we target. For example, the inference cost of a BERT-base model is approximately the same, no matter how it is parameterized. This uniform cost assumption also underpins the use of fixed cost penalties $c_i$ in our MDPs.

% Note that the last level of the cascade cannot further defer and would always decide to output a prediction. 

% \begin{algorithm}[t]
% \caption{Online Cascade Learning.}
% \label{algo:cascade}
% \begin{algorithmic}
% % \STATE \textbf{Input}: Models $M_i$, Query set $Q$
% \STATE Initialize $\mathcal{D} \leftarrow \emptyset$, policy weights $\hat{\pi}_1$, and $\beta_1$
% \FOR {$t=1$ to $T$}
% \STATE Sample query $x_t$ from $X$
% \FOR {$i=1$ to $N$}
% \STATE Let $\pi_t= \beta_t \pi^* + (1-\beta_t)\hat{\pi}_t$ by setting \\\ \ $\pi_t \leftarrow \pi^*$ at a decaying probability $\beta_{t}$
% \STATE $a^i_t = \pi_t(s_i)$
% \STATE \textbf{Continue if} $a^i_t = \texttt{defer}$ 
% \STATE $y_t \leftarrow a^i_t$
% \STATE $\mathcal{D} \leftarrow \mathcal{D} \cup \{x_t, y_t\}$ \textbf{if} $i$ \text{reaches} $N$ or $\pi_t = \pi^*$
% \STATE \textbf{Break} 
% % \STATE $C \leftarrow C + c^{x_i}_j$
% % \STATE \textbf{Break if} $\tau^i_j \equiv \text{max}(\phi_j) \geq t_k$
% \ENDFOR
% \STATE Output $y_t$ 
% % \STATE Compute $J(\pi_t)$ for the current episode
% \STATE Compute $J(\pi_t)$ on $\mathcal{D}$ and update $\hat{\pi}_{t+1}$ via OGD
% \ENDFOR
% \end{algorithmic}
% \end{algorithm}

% Therefore, by collecting the annotations from an expert policy as the approximation of ground truth labels, we can optimize both models (\ie $m_1, ..., m_N$) and deferral policies (\ie $p(\pi, s_1)', ..., p(\pi, s_{N-1})'$) together by minimizing the above objective function via online gradient descent. 

\section{Learning Algorithm}

To learn cascades online without human annotations, we propose an imitation learning algorithm that assumes an \emph{expert policy} (an LLM) that can demonstrate ground-truth labels. 
The algorithm iteratively updates the classification models and deferral functions in the cascade by imitating the expert as in DAgger \cite{ross2011reduction}. However, unlike traditional imitation learning, our goal here is to balance computational efficiency and accuracy. 

We incorporate such an expert into our cascades by assuming that the \emph{final} classification model
$m_N$ is the expert LLM. When invoked on a query $x_t$, $m_N$ always outputs a vector whose largest value is associated with ground-truth label $y_t$. However, it may not always be the case that choosing to invoke $m_N$ leads to optimal cost. There may be a smaller model $m_i$ for $i < N$ that also produces the correct label $y_t$, without the cost of the additional $\texttt{defer}$ actions. Indeed, an optimal policy may occasionally incur prediction errors to manage the cost of incurring $\texttt{defer}$ actions. Designing an algorithm that can train a policy in an online fashion to manage those trade-offs is at the core of the paper.

% The overall process of the algorithm is detailed in Algorithm \ref{algo:cascade}, where $\pi^*$ refers to the LLM expert for imitation, and $\hat{\pi}_1$ indicates the initial states of the small models. In practice, the models are initialized either randomly or with their respective pretrained weights. For every incoming query $x_t$ at episode $t$ (outer loop), we sequentially utilize the $i$-th model in the current cascade to make a prediction (inner loop). Specifically, the model outputs a class probability distribution, which is converted by the policy into a deferral probability score to determine the action $a^i_t=\pi_t(s_i)$ .

Our overall algorithm is detailed in Algorithm \ref{algo:cascade}. Here, $m_1, ..., m_{N}$ are the classification models (with $m_N$ being the expert for imitation) and $f_1, ..., f_{N-1}$ are the deferral functions. The smaller models $m_1, ..., m_{N-1}$ 
%are smaller in size and 
are initialized either randomly or with their respective pretrained weights.
%, and $m_N$ is the LLM expert for imitation. 
For each incoming query $x_t$ (outer loop), we sequentially utilize the $i$-th model in the cascade to make a prediction (inner loop). Specifically, the model generates a probability vector $\textit{pred}_i$, which is processed by the deferral function $f_i$. This yields a deferral probability score that can be discretized into a binary $\textit{action}_i$ to determine whether the prediction at the current cascade level is reliable.

If the action is $\texttt{defer}$, the query would be navigated to the cascade's next $(i+1)$-th model. Otherwise, it will make a prediction $\hat{y}_t = \argmax(\textit{pred}_i)$, and break the inner loop (succeeding models will not be activated for current query $x_t$). If the query has been deferred to the LLM expert $m_N$ at the last cascade level, its annotation $\hat{y}_t$ is regarded as the ground truth $y_t$ and aggregated to the dataset $\mathcal{D}$.

Throughout the inner iteration, at cascade level $i$, it may optionally skip the rest layers and jump to the LLM expert $m_N$ at a non-zero decaying probability $\beta_i$ to directly obtain its demonstration and aggregate it to dataset $\mathcal{D}$, similar to DAgger \cite{ross2011reduction}, for faster convergence. 

\begin{algorithm}[t]
\caption{Online Cascade Learning.}
\label{algo:cascade}
\begin{algorithmic}
% \STATE \textbf{Input}: Models $M_i$, Query set $Q$
\STATE Initialize models $m_1, ..., m_N$, deferral functions $f_1, ..., f_{N-1}$, $\beta_1$, $\mathcal{D} \leftarrow \emptyset$
\FOR {$x_t$ in stream $X$}
\FOR {$m_i$ in $m_1$ to $m_N$}
\STATE At probability $\beta_{t}$:
\STATE \ \ \ \ \textbf{Jump} to $m_N$ \textcolor{purple}{// \textit{like DAgger}}
\STATE $\textit{pred}_i = m_i(x_t)$  \textcolor{purple}{// \textit{ probability vector}}
\STATE $\textit{action}_i \sim f_i(\textit{pred}_i)$ \textcolor{purple}{// \textit{whether to defer}}
\IF{$m_i$ is $m_N$ \textbf{or} $\textit{action}_i \neq \texttt{defer}$}
\STATE $\hat{y}_t = \argmax(\textit{pred}_i)$ \textcolor{purple}{// \textit{discretization}}
\STATE $\mathcal{D} \leftarrow \mathcal{D} \cup \{x_t, \hat{y}_t\}$ \textbf{if} $m_i$ is $m_N$ 
\STATE \textbf{Break}
\ENDIF
\ENDFOR
\STATE Output $\hat{y}_t$
\STATE Update $m_1$ to $m_{N-1}$ on $\mathcal{D}$ via OGD \textcolor{purple}{// \textit{imitating expert}}
\STATE Compute loss $J(\pi, t)$ and update $f_1$ to $f_{N-1}$ via OGD
\STATE Decay $\beta_{t+1}$
\ENDFOR
\end{algorithmic}
\end{algorithm}

After processing each query, the algorithm outputs the prediction $y_t$, then updates the small models $m_1, ..., m_{N-1}$ to mimic the expert demonstrations on the collected trajectories $\mathcal{D}$. Similarly, the deferral functions $f_1, ..., f_{N-1}$ are also updated based on the loss computed by Equation (\ref{eq:loss}). The algorithm continuously collects annotations from the LLM expert (\eg at a decaying probability $\beta_t$ or when the query is deferred to $m_N$) and updates the policy via online gradient descent (OGD). Practically, the user can change the cost weighting factor $\mu$ in the loss function $J(\pi)$ and the initial decaying factor $\beta_1$ for adjusting cost budgets. 

% At the initial iterations (\ie the warm-up stage), it always defers to use the expert policy (\ie the LLM) to gather an annotated dataset, denoted as $\mathcal{D}$. Then, it trains a new policy $\hat{\pi}$ to mimic the expert on the trajectories in $\mathcal{D}$. At iteration $t$, the algorithm collects more trajectories using $\hat{\pi}_t$ and obtains expert policy annotation at a decaying probability $\beta_t \neq 0$, in order to better leverage the presence of the expert policy and improve robustness. 

\paragraph{Theoretical Analysis}
\label{sec:theory}
In online learning, a policy's \textit{regret} over time $T$ is its total cost minus the cost of the best fixed policy in hindsight. 
% Let $\Pi$ denote the possible policy space. 
In our setting, the regret of a learned policy $\pi$ is: 
\begin{align}
\gamma&= J(\pi, T) - \min\limits_{\pi \in \Pi} J(\pi, T)\\
&=\sum^{T}_{t=1} \sum^{N}_{i=1} p^{s_{t,i}}_{\pi} C_\pi(s_{t,i}) - \min\limits_{\pi\in\Pi}\sum^{T}_{t=1} \sum^{N}_{i=1} p^{s_{t,i}}_{\pi}C_\pi(s_{t,i}),
\label{eq1}
\end{align}
where $\Pi$ denotes the whole possible policy space. An algorithm is defined to have \emph{no-regret} if it can produce a sequence of policies $\pi_1, \dots, \pi_T$ such that the expected average regret $\gamma/T $ goes to $0$ as $T \rightarrow \infty$ \cite{ross2011reduction}. 

To aid our no-regret analysis of online cascade learning, we start by constructing a simplified \textit{online ensemble learning} algorithm under the same stream processing setting that comprises the linear combination of a series of classification models $m_1, ..., m_N$, each with a static operating probability  $\sum_{i=1}^{N}w_i=1$, without any deferral functions. Let us denote the model parameters of $m_i$ at time $t$ by $m_i^t$. Assuming a convex, differentiable cost function $c^t$ for all $t$ that can evaluate $m_i^t$, and bounded, closed, nonempty model spaces $||M_i||$ for all $m_i$, we analyze the regret of this algorithm. 

\begin{restatable}[]{theorem}{oel}
With online gradient descent and a learning rate $\eta_t = t^{-1/2}$, the total regret $\gamma$ of the \textbf{online ensemble learning} algorithm is bounded as follows:    
\begin{align}
\label{eq:ensemble_regret}
\gamma&=\sum^T_{t=1}\sum^N_{i=1}w_i \cdot c^t(m^t_i) - \min\limits_{m_i\in M^i}\sum^T_{t=1}\sum^N_{i=1}w_i\cdot c^t(m^t_i) \nonumber \\
&\leq \frac{||M||^2\sqrt{T}}{2} + (\sqrt{T}-\frac{1}{2}) ||\nabla c||^2.
\end{align}
Therefore, $\lim_{T\rightarrow \infty} \gamma/T \leq 0$.
\label{proof:ensemble}
\end{restatable} 
The proof of Theorem \ref{proof:ensemble} can be constructed as an extension of Theorem 1 in \citet{zinkevich2003online}. Here, $||M||$ denotes the maximum distance within model spaces, and $||\nabla c||$ represents the largest gradient magnitude of the cost function across models. We defer its proof to Appendix \ref{app:ensemble}.

\begin{restatable}[]{theorem}{ocl}
For \textbf{online cascade learning} with $\eta_t = t^{-1/2}$, the algorithm's total regret is $o(T)$, implying $\lim_{T\rightarrow\infty} \gamma/T \leq 0$, i.e., the average regret approaches zero as $T$ grows $\infty$.
\label{proof:cascade}
\end{restatable}
Having established the no-regret property for online ensemble learning, we extend this to our online cascade learning algorithm. 
By replacing the static probabilities $w_i$ in Equation (\ref{eq:ensemble_regret}) with dynamic probabilities based on preceding model actions, and demonstrating convergence of these dynamic probabilities to optimal values as $T\rightarrow\infty$ (proof in Appendix \ref{app:lemma}, Lemma \ref{proof:lemma}), we mirror the total regret of online cascade learning to that of the online ensemble learning algorithm, completing the proof for Theorem \ref{proof:cascade}. Further details are in Appendix \ref{app:cascade}.

\paragraph{Confidence Calibration}
\label{sec:conf}
A reliable deferral rule is crucial for a cascade system to determine whether to invoke the next model (\ie to defer) or to output the predictions. Currently, most existing works make deferral decisions based on a confidence score, typically measured either by the maximum predictive probability across all classes \cite{DBLP:conf/iclr/WangKCKME22, DBLP:conf/emnlp/VarshneyB22}, or the Shannon entropy of the predictive distribution \cite{stogiannidis2023cache}.

% However, for online cascade learning, where the model capabilities are dynamically updated, and the annotations from LLMs may be noisy, confidence-based deferral policies have been shown to be inadequate \cite{confidence}. To aid the learning of deferral policies (\ie $p(\pi, s_i)'$), we adopt a post-hoc deferral policy $\pi_i'$ to calibrate the confidence estimate of a prediction from model $m_i$. It is implemented using a multi-layer perceptron (MLP) that takes the corresponding model's predictive probabilities as input, and the policy can be learned with the following objective:
However, for online cascade learning, where the model capabilities are dynamically updated, and the annotations from LLMs may be noisy, confidence-based deferral rules have been shown to be inadequate \cite{confidence}. To aid the learning of deferral functions (\ie $f_1, ..., f_{N-1}$), we adopt a post-hoc approach to calibrate the confidence estimate of a certain model's prediction $m_i(x_t)$. It is implemented using a multi-layer perceptron (MLP) that takes the corresponding model's predictive probabilities as input, and the functions can be updated with the following objective:
% \begin{equation}
%     \min_{\pi_i': \mathbb{R}^{|Y|}\rightarrow \mathbb{R}}\sum_{x_t \in X}\sum_{i=1}^{N-1} L(\pi'_i(m_i(x_t)), z_i),
% \end{equation}
\begin{equation}
    \min_{\pi_i': \mathbb{R}^{|Y|}\rightarrow \mathbb{R}}\sum_{x_t \in X}\sum_{i=1}^{N-1} L(f_i(m_i(x_t)), z_i),
\end{equation}
where $z_i = \mathbbm{1}[\argmax m_i(x_t) \neq y^*_t]$ (\ie $z_i =1$ if model $m_i$'s prediction is not equal to the annotation $y^*_t$, otherwise $z_i =0$) and $L$ is the mean-squared error loss function.

Since we treat the expert LLM predictions (\ie $\argmax(m_N(x_t))$) as the ground truth labels $y_t$, calibration is only performed on those input queries where the expert LLM is invoked. During the calibration, the learning of $f_i$ would only optimize the parameters of the MLP, not the models $m_i$ in the cascade. During inference time, the post-hoc deferral function would choose to \texttt{defer} if $f_i(m_i(x_t)) > 0.5$. Otherwise, it would output the prediction of the current model.

%% file: sections/3_exp.tex
\begin{table*}[t]
    \centering
    \resizebox{\textwidth}{!}{%
    \begin{tabular}{lccc|ccc|ccc|ccc}
    \toprule
         &  \multicolumn{3}{c}{\textbf{IMDB}} &  \multicolumn{3}{c}{\textbf{HateSpeech} {\small(Accuracy  $|$  Recall)}}  &  \multicolumn{3}{c}{\textbf{ISEAR}} & \multicolumn{3}{c}{\textbf{FEVER}}\\
    \cmidrule(lr){2-4}\cmidrule(lr){5-7}\cmidrule(lr){8-10}\cmidrule(lr){11-13}
    & $\mathcal{N}$=1300         & $\mathcal{N}$=3800         & $\mathcal{N}$=5200         & $\mathcal{N}$=600          & $\mathcal{N}$=2700         & $\mathcal{N}$=4900    
& $\mathcal{N}$=1200         & $\mathcal{N}$=1500         & $\mathcal{N}$=2700          & $\mathcal{N}$=700         & $\mathcal{N}$=2000  & $\mathcal{N}$=2800              \\ \midrule
    \textbf{GPT-3.5 Turbo}     & \multicolumn{3}{c|}{94.15} &  \multicolumn{3}{c|}{83.34\hspace{1em} $|$\hspace{1em} 83.28}& \multicolumn{3}{c|}{70.34} & \multicolumn{3}{c}{79.98} \\
    \midrule
    Distilled LR      & 82.61 & 83.60 & 87.01 & 80.18 $|$ 37.94  & 82.23 $|$ 49.25 & \textbf{85.03} $|$ 45.59 & 44.97 & 47.46 & 48.92 & 56.51 & 57.80 & 57.13\\
    Distilled BERT-base    & 85.28 & 90.18 & 90.19 & 80.49 $|$ 64.39 & 80.71 $|$ 73.88 & 79.35 $|$ 77.37 & \textbf{61.49} & 62.62 & 63.37 & 61.70 & 63.64 & 70.82 \\
     Online Ensemble Learning  &  86.73 & 88.80 & 89.95 & 82.61 $|$ 76.75 & 77.48 $|$ 76.89 & 81.55 $|$ 80.30 & 56.56 & 60.42 & 61.78
 & 61.69 & 69.78 & 76.67 \\
    \textbf{Online Cascade Learning}  &  \textbf{87.95} & \textbf{92.48} & \textbf{93.01}  & \textbf{82.66 $|$ 82.36} & \textbf{85.35} $|$ \textbf{77.20} & 83.26 $|$ \textbf{81.03} & 60.78 & \textbf{65.34} & \textbf{69.75} & \textbf{61.95} & \textbf{71.86} & \textbf{78.49}  \\
    \bottomrule
    \toprule
    % \cmidrule(lr){2-4}\cmidrule(lr){5-7}\cmidrule(lr){8-10}\cmidrule(lr){11-13}
%     & $\mathcal{N}$=1300         & $\mathcal{N}$=3800         & $\mathcal{N}$=5100         & $\mathcal{N}$=600          & $\mathcal{N}$=3300         & $\mathcal{N}$=4900    
% & $\mathcal{N}$=1200         & $\mathcal{N}$=2300         & $\mathcal{N}$=3500          & $\mathcal{N}$=900         & $\mathcal{N}$=2300  & $\mathcal{N}$=3000              \\ \midrule
    \textbf{Llama 2 70B Chat}     & \multicolumn{3}{c|}{93.33} &  \multicolumn{3}{c|}{77.81\hspace{1em} $|$\hspace{1em} 82.19}& \multicolumn{3}{c|}{68.23} & \multicolumn{3}{c}{77.15} \\
    \midrule
    Distilled LR      &  82.17 & 85.80 & 86.88 & 67.94 $|$ 66.56 & \textbf{79.71} $|$ 61.73 & \textbf{81.46} $|$ 49.91 & 46.78 & 47.56 & 51.76 & 57.46 & 61.24 & 58.42 \\
    Distilled BERT-base    & 85.39 & 85.59 & 85.44 & 75.84 $|$ \textbf{78.87} & 79.18 $|$ 75.54 & 80.27 $|$ 72.21 & \textbf{62.18} & 61.84 & 65.12 & \textbf{65.88} & 65.66 & 67.54 \\
     Online Ensemble Learning  &  87.14 & 88.66 & 89.61
 &  75.99 $|$ 60.36 & 70.79 $|$ \textbf{79.16} & 76.82 $|$ 81.84 & 54.74 & 57.35 & 60.19 & 63.48 &71.27 & 76.46 \\
    \textbf{Online Cascade Learning}  & \textbf{87.58} & \textbf{92.14} & \textbf{92.63}  & \textbf{78.30} $|$ 63.06 & 78.32 $|$ 76.54 & 78.32 $|$ \textbf{82.03} & 59.24 & \textbf{63.34} & \textbf{67.25} & 63.81 & \textbf{72.47} & \textbf{77.73}  \\
    \bottomrule
    \end{tabular}
    }
    \caption{Comparison of accuracy (and recall for HateSpeech dataset) among different methods under various cost budgets. The upper part of the table uses GPT-3.5 Turbo as the LLM in the cascade, while the lower part employs Llama 2 70B Chat. To ensure fairness, the same annotation cost budgets (\ie the maximum allowable LLM calls, denoted as $\mathcal{N}$, controlled via adjusting the cost weighting factor $\mu$ and decaying factor $\beta$ for online cascade learning) are applied across all methods.}
    \label{tab:main}
\end{table*}

\section{Experimental Setup}
%In our experiments, 
We construct a cascade system using three models: (i) a logistic regression model, (ii) a pretrained BERT-base model with 110M parameters \cite{kenton2019bert}, and (iii) GPT-3.5 Turbo\footnote{\url{https://platform.openai.com/docs/models/gpt-3-5}}. This diverse set of models allows us to evaluate the effectiveness and robustness of our %online cascade learning 
approach across various complexities and types of queries. 

To further test the adaptability and scalability of our system, we also conduct supplementary experiments where (a) a Llama 2 70B Chat \cite{touvron2023llama} is used as the alternative LLM, and (b) a BERT-large with 340M parameters is incorporated to create a larger cascade. These variations aim to demonstrate the flexibility of our framework in accommodating different cascade sizes and structures. 

%\subsection{Benchmarks}
\paragraph{Benchmarks}
We evaluate our approach on four benchmarks that reflect the demands of real-world streaming applications
%. These tasks are valuable 
in a wide range of commercial services, from customer feedback analysis to content moderation:
%, and 
%thus offering a comprehensive testbed to evaluate our framework's adaptability and efficiency in various practical scenarios. 
\begin{itemize}[leftmargin=0.2in]
\item \textit{\textbf{IMDB.}} A binary sentiment classification benchmark with 50,000 movie reviews \citep{imdb}. The dataset has an even distribution of \texttt{positive} and \texttt{negative} samples. We use the official training split for our experiments, which contains 25,000 samples.

\item \textit{\textbf{HateSpeech.}} A binary classification dataset consisting of posts from an online forum, annotated with \texttt{hate} and \texttt{noHate} labels \cite{gibert2018hate}. After filtering, the dataset contains 10,703 samples with a pronounced class imbalance (1:7.95 ratio) between hatespeech and non-hatespeech examples. This imbalance mirrors a realistic challenge in streaming data environments, particularly in detecting harmful content.  Our evaluation on this benchmark focuses on both \textit{accuracy} and \textit{recall}. 

\item \textit{\textbf{ISEAR.}} A multi-class emotion detection benchmark encompassing 7,666 samples across seven categories (\texttt{Joy}, \texttt{Fear}, \texttt{Anger}, \texttt{Sadness}, \texttt{Disgust}, \texttt{Shame}, \texttt{Guilt}) \cite{isear}. Each category is well-represented, providing a balanced label distribution across the dataset. 

\item \textit{\textbf{FEVER.}} A fact-checking dataset with 6,512 claims manually verified against Wikipedia, labeled as \texttt{Supported} or \texttt{Refuted} \cite{thorne2018fever}. It tests our framework's ability to perform complex reasoning and information verification, a crucial aspect for real-time truth assessment in streaming data applications.
\end{itemize}

\begin{figure*}[h!]
    \centering
{\includegraphics[height=95pt]{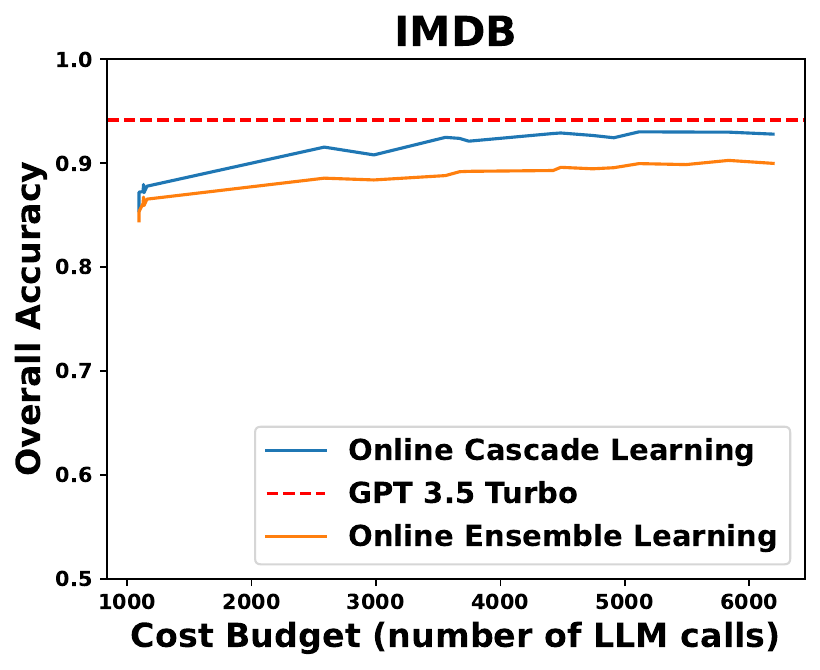}}
{\includegraphics[height=95pt]{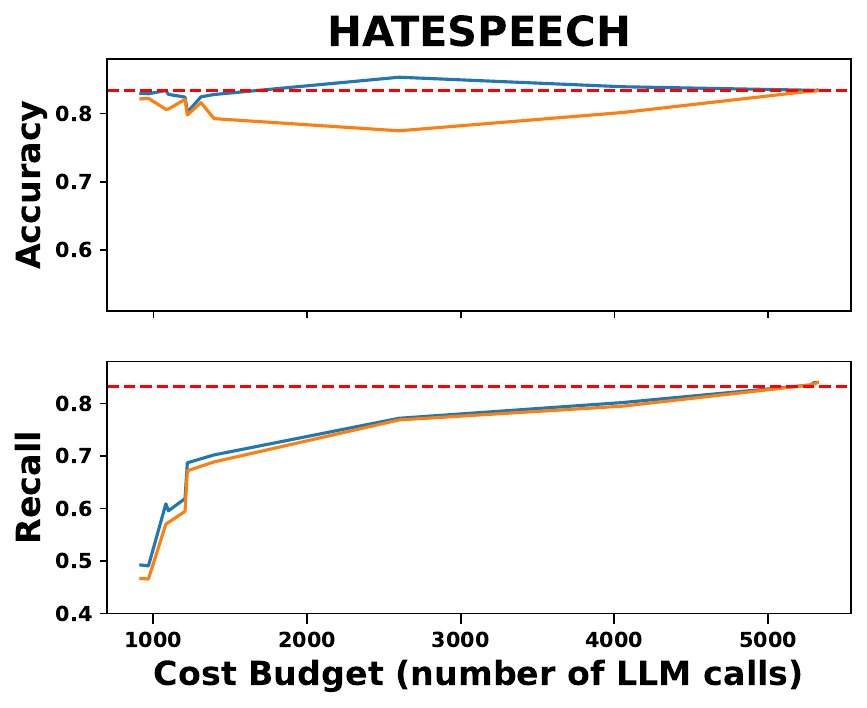}}
{\includegraphics[height=95pt]{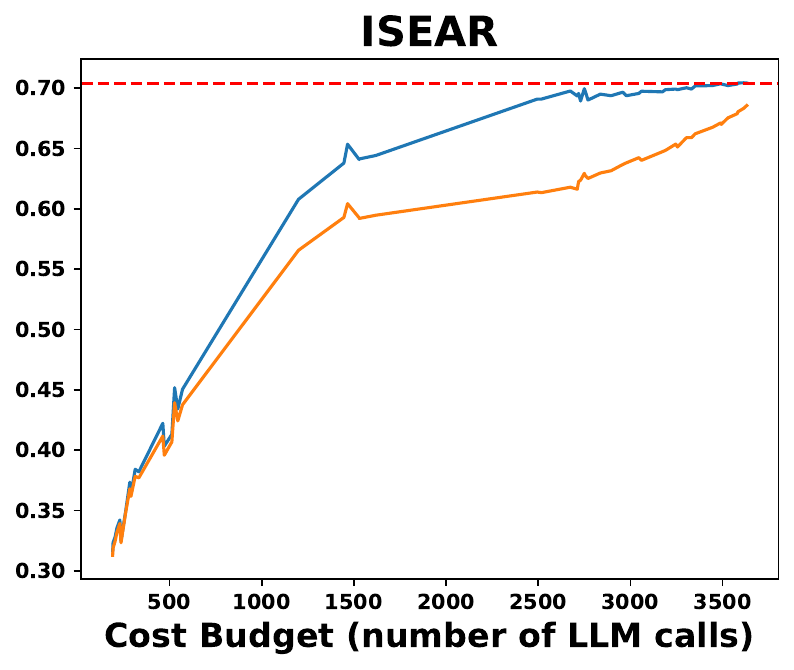}} 
{\includegraphics[height=95pt]{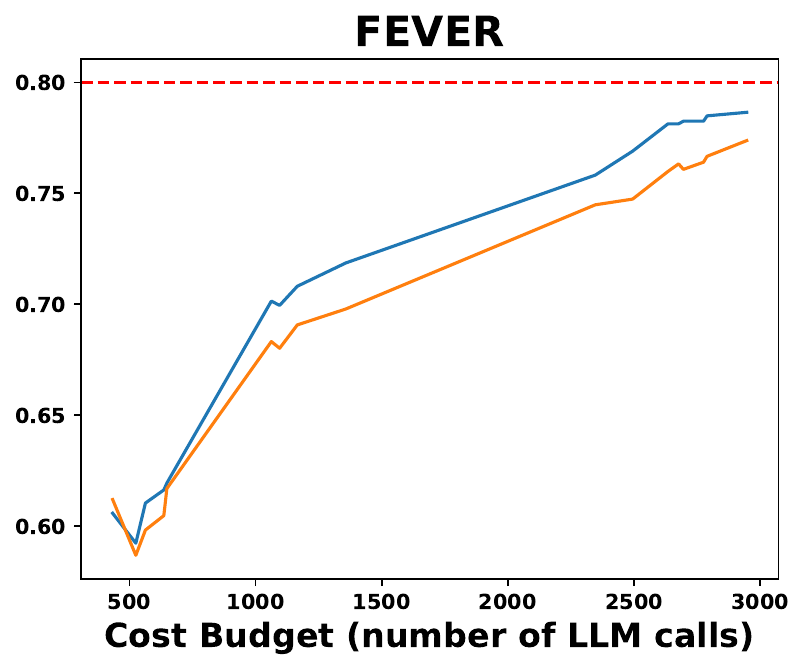}} 

    \caption{Accuracy curve (and Recall curve for HateSpeech) with respect to costs, using GPT-3.5 Turbo as the LLM in a cascade that also comprises logistic regression and BERT-base.}
    \label{fig:tradeoff_gpt}
\end{figure*}

\begin{figure*}[h!]
    \centering
{\includegraphics[height=95pt]{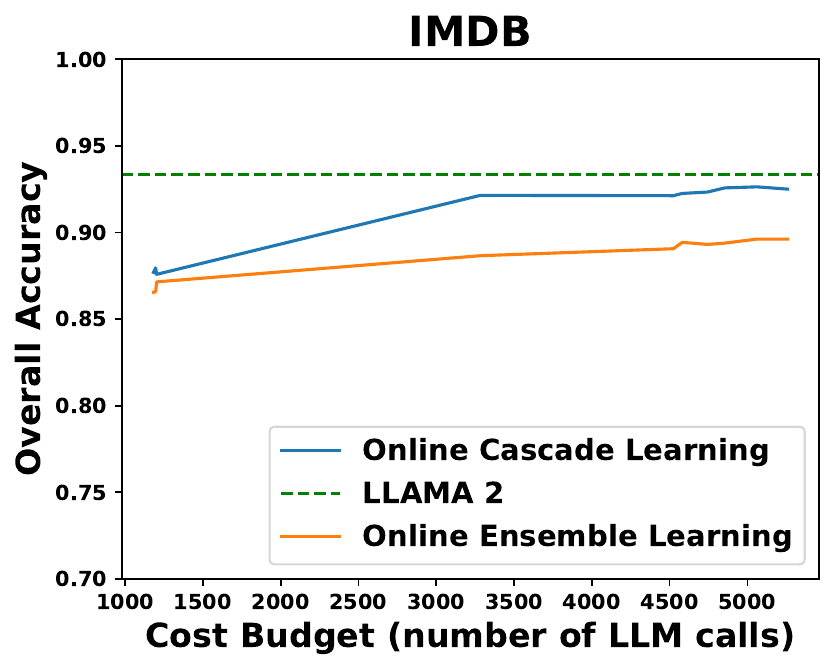}}
{\includegraphics[height=95pt]{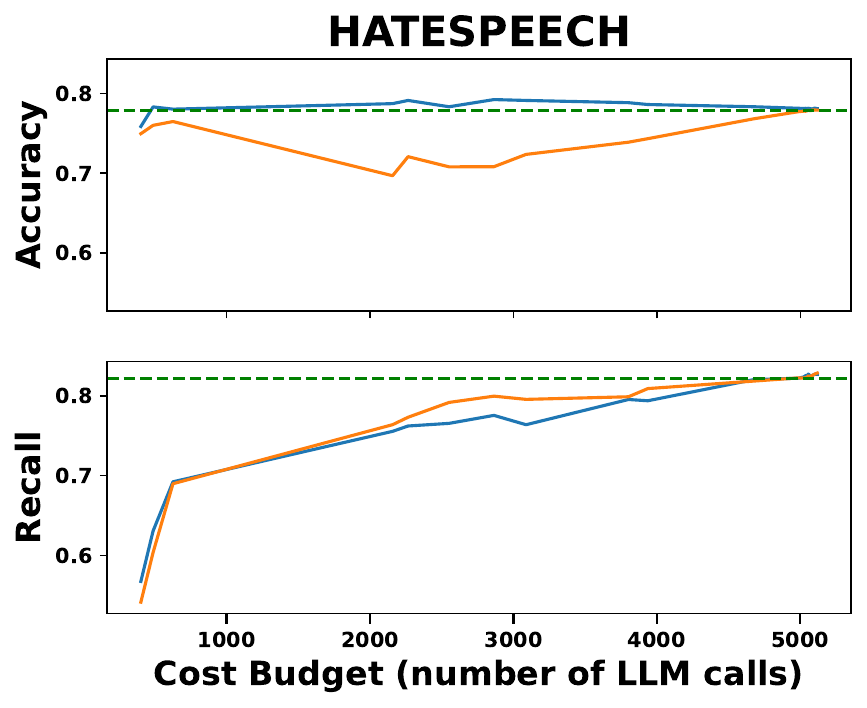}}
{\includegraphics[height=95pt]{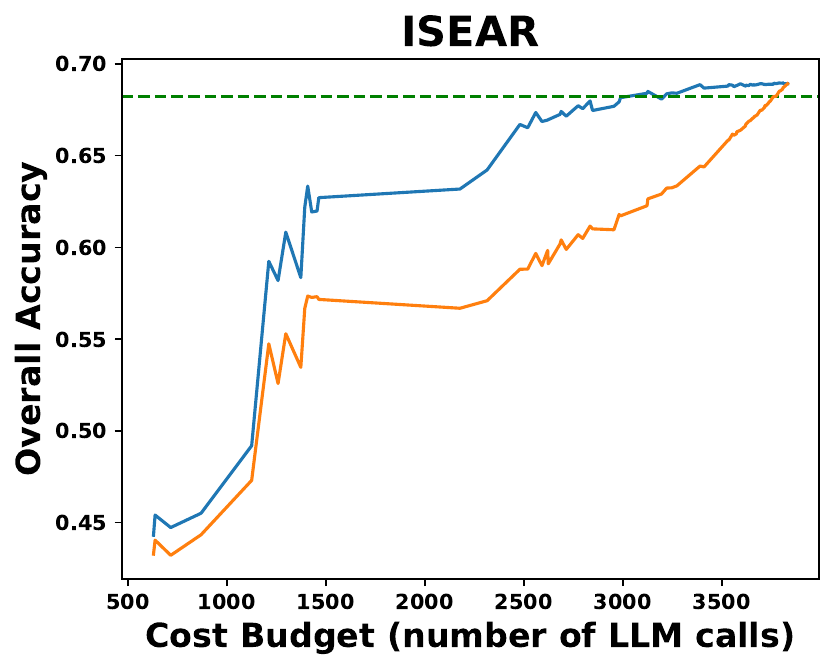}} 
{\includegraphics[height=95pt]{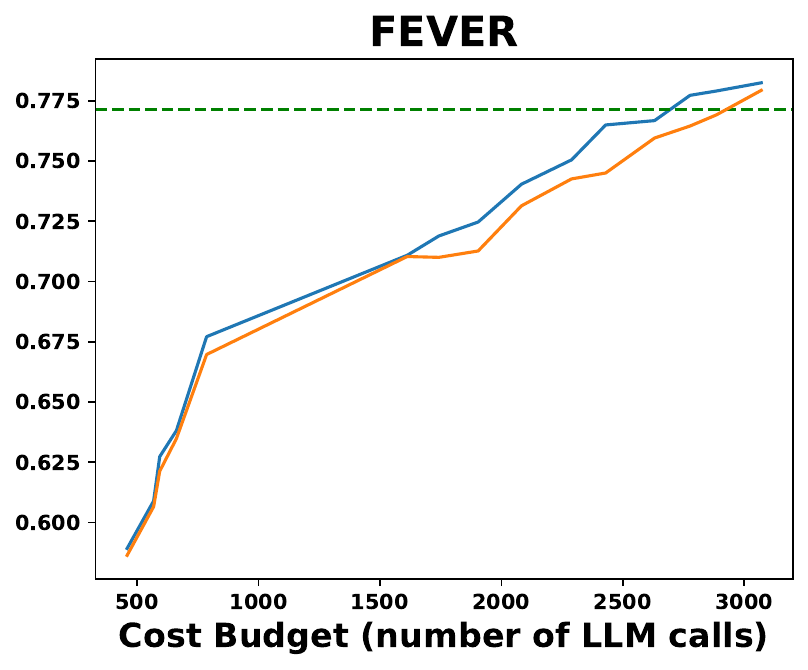}} 

    \caption{Accuracy curve (and Recall curve for HateSpeech) with respect to costs, using Llama 2 70B Chat as the LLM in a cascade that also comprises logistic regression and BERT-base. }
    \label{fig:tradeoff_llama}
\end{figure*}

\paragraph{Baselines}
\label{sec:baseline}
We compare online cascade learning against several baselines to establish its effectiveness:
%to establish its effectiveness:
\vspace{-\topsep}
\begin{itemize}[leftmargin=0.2in]
\item 
\textit{\textbf{LLMs in the Cascade:}} This includes GPT-3.5 Turbo and Llama 2 70B Chat with zero-shot task prompting (details in Appendix \ref{sec:prompt}). The LLM outputs are also used as the annotations for online cascade learning and distillation. 
    \item 
    \textit{\textbf{Knowledge Distillation:}} We fine-tune smaller models using different portions of LLM annotations. To ensure fairness, datasets are split equally, with 50\% prepared for training (as distillation labels) and the remaining 50\% for testing. All methods are evaluated on the identical test sets. In our experiments, the distilled smaller models are used in isolation without any ensemble or cascade. 
    
    \item 
    \textit{\textbf{Online Ensemble Learning:}} We employ all available models in an ensemble with learned predetermined probabilities. The smaller models are also continuously updated based on LLM annotations. This serves as an ablation of our method by excluding the deferral policy learning component.
    % \item Offline model cascade, by fine-tuning the smaller models offline using the annotations collected from LLM in advance and keeping all the models fixed in the inference time \cite{DBLP:conf/emnlp/VarshneyB22}.
\end{itemize}

\section{Experimental Results}
\subsection{Overall Performance and Cost Trade-offs}
\paragraph{IMDB} The results presented in Table \ref{tab:main}, Figure \ref{fig:tradeoff_gpt} and \ref{fig:tradeoff_llama} demonstrate that our proposed online cascade learning system can consistently achieve higher accuracies compared to knowledge distillation and online ensemble learning baselines on the IMDB dataset across different cost budgets, regardless of whether using GPT-3.5 Turbo or Llama 2 70B Chat in the cascade. Notably, Figure \ref{fig:tradeoff_llama} highlights our system's ability to closely rival the performance of Llama 2 70B Chat while achieving a 60\% reduction in inference costs (\ie calling LLM $\sim$5200 times in processing a total of 12500 queries). This effectively demonstrates the system's efficiency in balancing cost with performance.

\paragraph{HateSpeech} The results on the HateSpeech dataset further reveal the strengths of our online cascade learning system in handling datasets with significant class imbalance. Most models may face a trade-off between accuracy and recall due to the imbalanced nature of HateSpeech. However, the accuracy-cost and recall-cost trade-offs, respectively depicted in the upper and lower subplots of Figure \ref{fig:tradeoff_gpt} and \ref{fig:tradeoff_llama}, demonstrate that our system effectively improves recall with minimal impact on accuracy as the cost budget increases. Although the recall rate of online cascade learning is marginally lower than the baselines under certain budgets, it can achieve a better balance between recall and precision, as evidenced by its consistently higher F1 scores in Appendix Figure \ref{fig:f1_gpt}. In particular, as demonstrated in Table \ref{tab:main} where Llama 2 70B Chat is the LLM and the cost budget $\mathcal{N}=4900$, our system even outperforms the LLM with a similar recall (ours: 82.03\% vs. LLM: 82.19\%) and a better accuracy (78.32\% vs. 77.81\%), underlining its effectiveness in handling imbalanced data streams.

\paragraph{ISEAR} On the ISEAR benchmark, our online cascade learning system also effectively balances cost and accuracy in complex multi-class classification. As indicated in Figure \ref{fig:tradeoff_gpt} and \ref{fig:tradeoff_llama}, the system's performance gradually aligns with that of GPT-3.5 Turbo and even surpasses Llama 2 70B Chat as the cost budget increases. This success underscores the advantages of smaller models in adapting to complex classifications by learning from the LLM annotations, enabling them to potentially outshine zero-shot LLMs. Moreover, the notable performance gap between online ensemble learning and online cascade learning also confirms the benefit of co-optimizing model learning with deferral policy learning for optimal cost-performance equilibrium.

\paragraph{FEVER} FEVER is a significantly more complex dataset compared to the previous benchmarks. It demands models to reason over the statements and validate their factuality based on parametric knowledge. Therefore, small models of limited capacities, such as logistic regression, struggle to perform effectively on FEVER even after several iterations of update, as evident in Table \ref{tab:main} where distilled LR can perform only slightly better than random guess (\ie 50\%). Recognizing these limitations, our online cascade learning system smartly adapts by prioritizing more capable models, such as BERT-base and the LLM, for processing most of the queries, leading to a favorable accuracy-cost trade-off. Remarkably, when using Llama 2 70B Chat as the LLM with a cost budget of $\mathcal{N}=2800$, our system slightly outperforms the LLM in accuracy (77.73\% vs. 77.15\%), showcasing the system's proficiency at navigating intricate reasoning tasks with enhanced cost-efficiency.

\subsection{Case Analysis}
To examine our approach's online learning process more closely, we run the online cascade learning at specific cost budgets and conduct a detailed case analysis.

\paragraph{IMDB} Figure \ref{fig:imdb_exp} illustrates the online cascade system performance at a specific cost budget ($\mathcal{N}=3671$) throughout the inference of the IMDB dataset. Initially, for the first 160 samples, all queries are processed exclusively by the LLM, as indicated by the stacked plot in the background. However, with the arrival of more queries over time, both the logistic regression and BERT-base models, denoted by the green and orange dashed lines, dynamically improve by learning from the GPT-3.5 Turbo annotations and increasingly contribute to query processing. When the incoming number of samples approaches 4000, most queries are processed by BERT-base, leaving only 30\% of the queries deferred to GPT-3.5 Turbo. Meanwhile, the system can achieve an overall accuracy consistently close to or even slightly higher than the LLM (marked by the blue dashed line), which affirms the effectiveness of online cascade learning in saving inference costs with minimum performance degradation. 

\begin{figure}[t]
     \centering
     \includegraphics[width=\columnwidth]{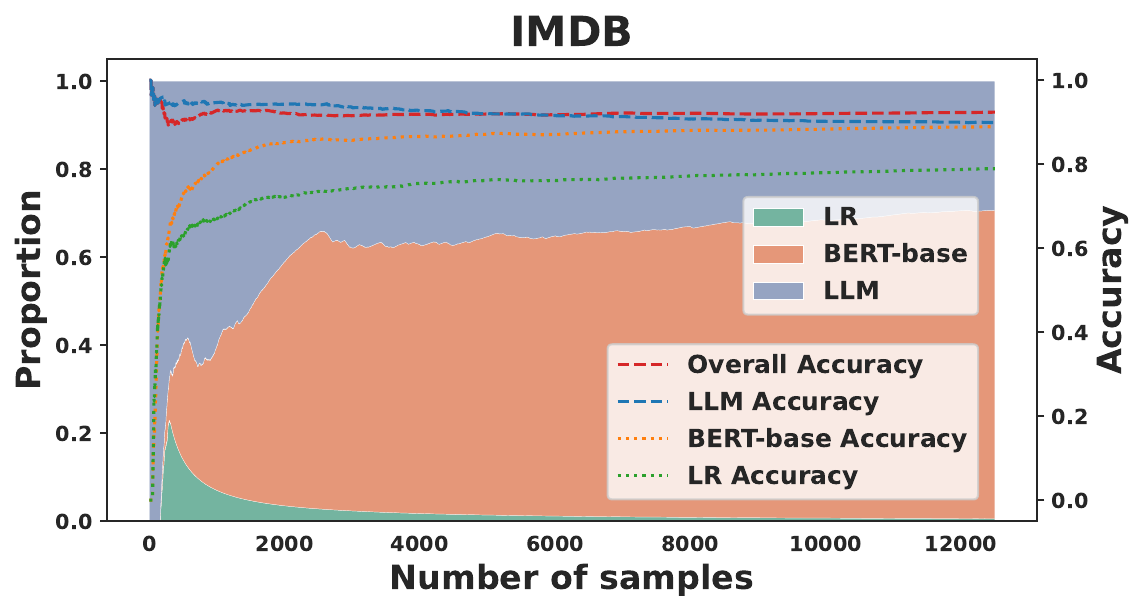}
     \caption{Inference results on IMDB when $\mathcal{N}=3671$. Online cascade learning system performs similarly to GPT-3.5 Turbo while saving $\sim$70\% of the inference costs.} 
     \label{fig:imdb_exp}
\end{figure}

\begin{figure}[t]
     \centering
     \includegraphics[width=\columnwidth]{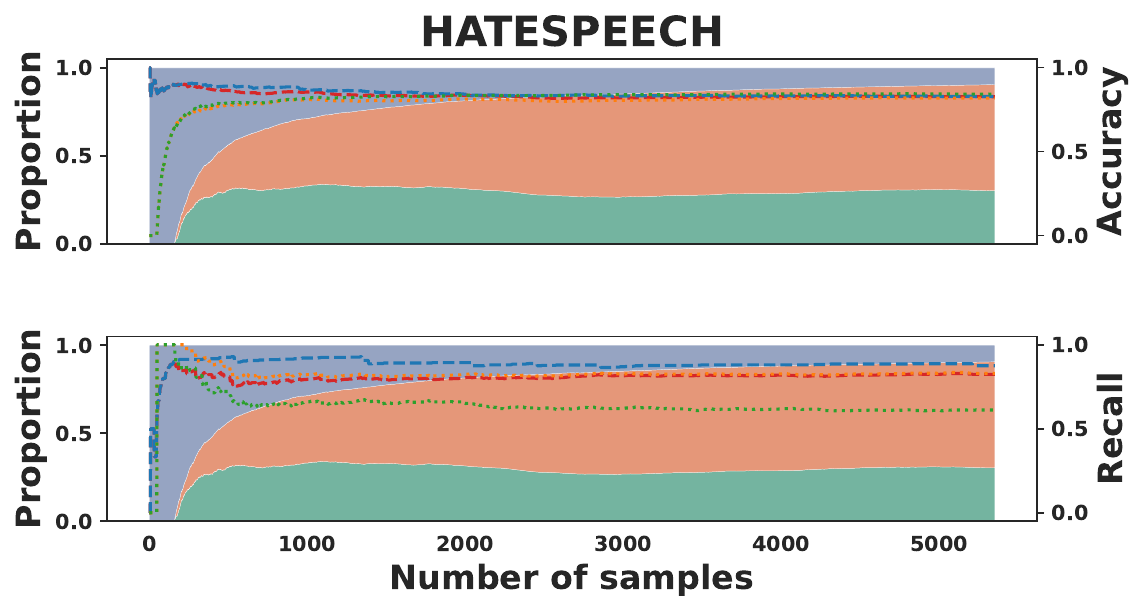}
     \caption{Inference results on HateSpeech when $\mathcal{N}=507$. Online cascade learning system performs similarly to GPT-3.5 Turbo while saving $\sim$90\% of the inference costs.} 
     \label{fig:hatespeech_exp}
\end{figure}

\begin{figure}[t]
     \centering
     \includegraphics[width=\columnwidth]{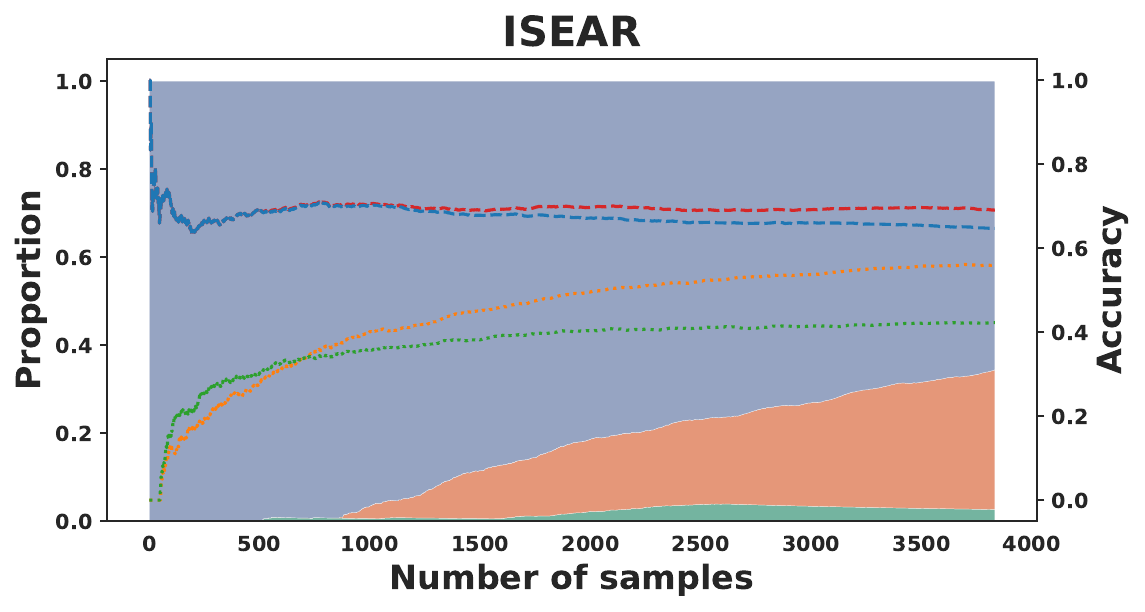}
     \caption{Inference results on ISEAR when $\mathcal{N}=2517$. Online cascade learning system performs very close to GPT-3.5 Turbo while saving $\sim$30\% of the inference costs.} 
     \label{fig:isear_exp}
\end{figure}

\begin{figure}[t]
     \centering
     \includegraphics[width=\columnwidth]{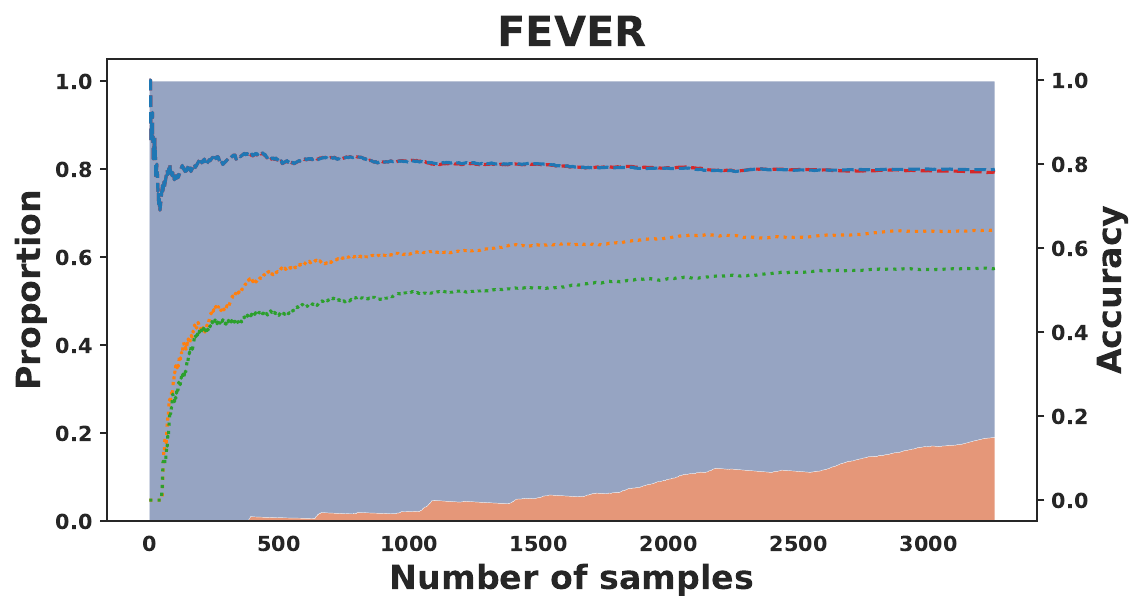}
     \caption{Inference results on FEVER when $\mathcal{N}=2635$. Online cascade learning system performs similarly to GPT-3.5 Turbo while saving $\sim$20\% of the inference costs. } 
     \label{fig:fever_exp}
\end{figure}

\begin{figure*}[t!]
    \centering
{\includegraphics[height=101pt]{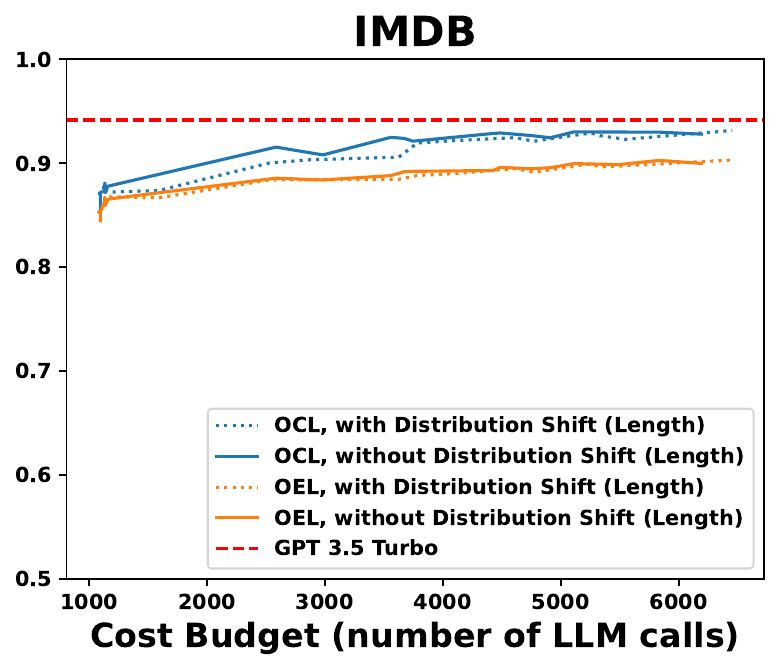}}
{\includegraphics[height=101pt]{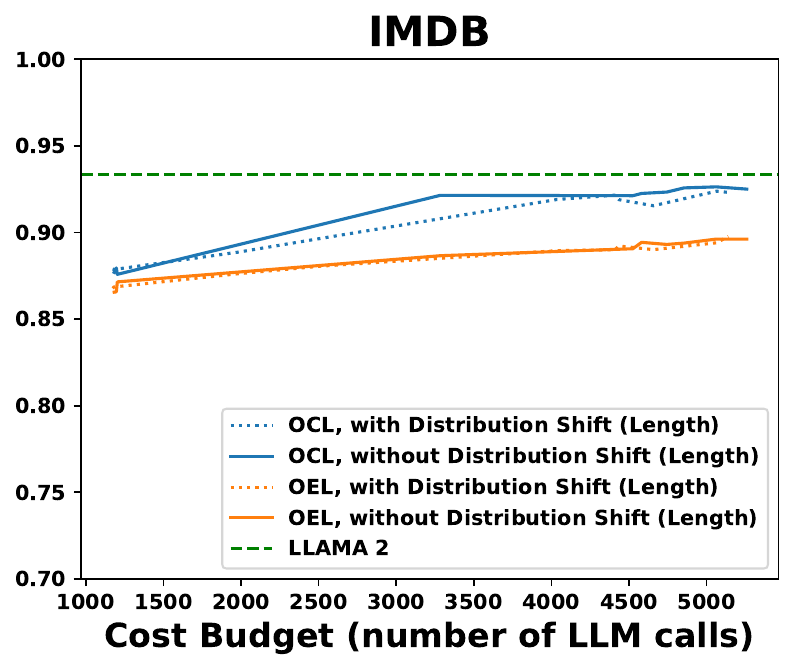}}
{\includegraphics[height=101pt]{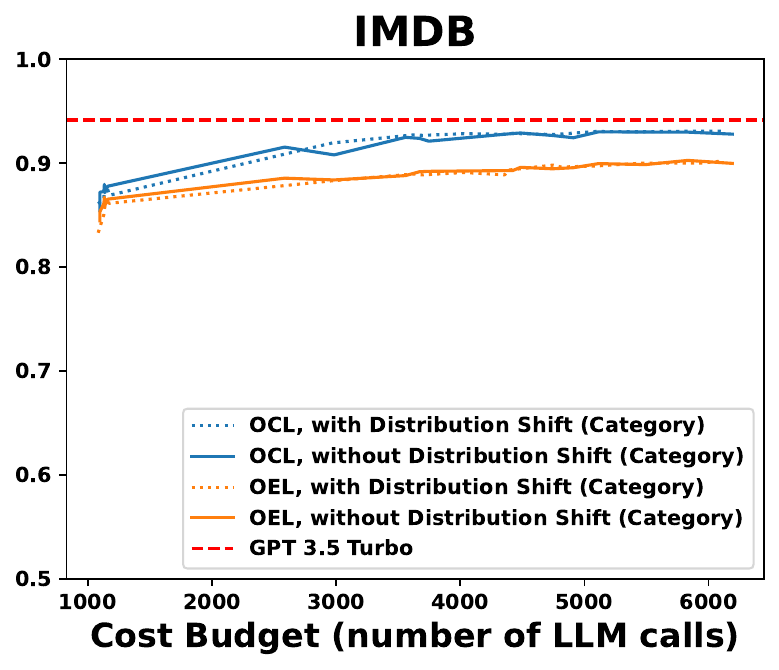}} 
{\includegraphics[height=101pt]{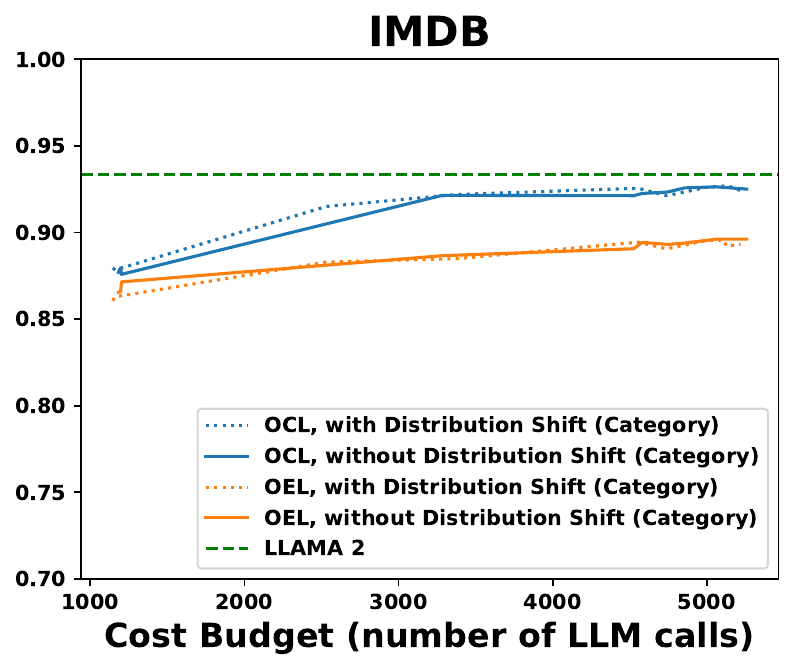}} 

    \caption{Cost-accuracy trade-off curves on two input distributional shift scenarios, respectively using GPT-3.5 Turbo and Llama 2 70B Chat as the LLM with logistic regression and BERT-base in a cascade. ``OCL'' refers to online cascade learning, and ``OEL'' means online ensemble learning. }
    \label{fig:shift}
\end{figure*}

\paragraph{HateSpeech} Similarly, on HateSpeech dataset, when $\mathcal{N}=507$, as shown in Figure \ref{fig:hatespeech_exp}, as the number of samples approaches 5000, 30.31\% of the queries are handled by logistic regression, while 60.24\% are handled by BERT-base, which altogether cut down the LLM inference costs by more than 90\%. At the same time, the online cascade system can still achieve an overall accuracy of 82.66\% and recall of 82.36\%, which aligns closely with the performance of GPT-3.5 Turbo (Accuracy: 83.34\%, Recall: 83.28\%)

\paragraph{ISEAR} Our analysis of the ISEAR dataset under a specific cost budget ($\mathcal{N}=2517$) also confirms the dynamic adaptability and cost-saving features of our online cascade system. As the number of processed samples increases, the system demonstrates an impressive ability to gradually shift query processing from the expensive GPT-3.5 Turbo to the more economical models. This transition is evident in the increasing proportions of queries handled by BERT-base over time, as shown in Figure \ref{fig:isear_exp}. Most importantly, since the proportion of queries handled by BERT-base is not yet converged, the whole system can further cut inference costs while sustaining high accuracy as more samples come in.

\paragraph{FEVER} On the FEVER dataset, our system's performance at a cost budget ($\mathcal{N}=3671$) further validates its usefulness even under complex reasoning task settings. The system's learning curve, as depicted in Figure \ref{fig:fever_exp}, illustrates a steady increase in the number of queries processed by the BERT-base models. Given logistic regression's limited capability in fact-checking, our online cascade learning system smartly shifts its reliance towards more capable models, such as BERT-base and the larger LLMs. The overall performance in terms of accuracy aligns closely with that of the GPT-3.5 Turbo, achieving significant cost savings (17\%) without compromising on the quality of the results.

\subsection{Adaptability to Larger Cascade}
\label{sec:larger}
% We explore the adaptability of our online cascade structure by integrating a more extensive BERT-large model into our cascade. The larger augmented cascade outperforms the smaller cascade on IMDB, ISEAR, and FEVER datasets, as depicted in Appendix Figure \ref{}.  The larger cascade trained on Llama 2 annotations can attain equivalent or superior performance to the standalone Llama 2 across all datasets except HateSpeech. This demonstrates our cascade's ability to adopt a more granular adaptation to the complexity of incoming queries, potentially improving both response quality and cost efficiency. Meanwhile, to optimize the trade-off between cost and accuracy, future iterations of the cascade system may consider incorporating LLMs of different sizes and specialties, allowing for refined tailoring of responses to the specific characteristics of different queries.

To further validate the adaptability of our online cascade learning system, we have also explored scaling up the cascade system by further integrating a BERT-large model (4 models in total). The results of this integration are encouraging, detailed in Appendix Figure \ref{fig:large_tradeoff_gpt}. The expanded cascade system, particularly when trained on Llama 2 70B Chat annotations, demonstrates equivalent or even superior results compared to the standalone LLM in most cases, with the exception of the HateSpeech dataset, a simpler task setting where a larger cascade might complicate the deferral policy learning and thus degrading overall performance. Therefore, it is important to align cascade size with task complexity to optimize performance and avoid overfitting. A more detailed analysis is in Appendix \ref{app:larger}
% This outcome underscores the cascade's adaptability and its capacity for granular response tailoring to match the complexity levels of incoming queries, 

The success of this larger cascade system signals the huge potential in our work's extensions. It showcases our system's inherent flexibility and scalability, enabling it to accommodate and efficiently utilize more powerful models as part of its architecture. Future investigations on online cascade systems may further scale up the system by incorporating LLMs of different sizes and specialties, allowing for refined tailoring of responses to the specific characteristics of different queries. However, there are also several factors worth noting when scaling up the cascade, as outlined in our analysis in Appendix \ref{app:equi} and \ref{app:llm}.

\subsection{Robustness against Distributional Shifts}

\paragraph{Distribution Shift in Input Length} 
Longer inputs typically involve more complex semantics. For example, on IMDB, the average accuracy of GPT-3.5 Turbo is notably lower on longer movie reviews, as evidenced in Appendix Table \ref{tab:length}.
%  the accuracy of GPT 3.5 Turbo is significantly higher on a group of movie reviews with an average length of 482, at 95.54\%, as compared to a group with an average length of 2954 where the accuracy is 92.44\%.
Therefore, to verify our method's robustness against input distribution shifts, we rearrange the IMDB benchmark with length ascending order to simulate a distribution shift over the inputs’ semantic complexity. The results are shown in Figure \ref{fig:shift} and Table \ref{tab:shift}. Despite the distribution shift in input length, online cascade learning demonstrates good performance with minimal accuracy drops across different cost budgets, regardless of the LLM adopted in the cascade.

\paragraph{Distribution Shift in Input Category} 
We further validate our approach with a distribution shift in input semantic categories, by filtering all the input samples regarding “\texttt{Comedy}” movies in the IMDB dataset (8,140 out of 25,000), and feeding them into the cascade as the last part of the input stream. This means that the system had not seen any comedy movie reviews in the first 2/3 of the inputs before processing comedy movie reviews in the last 1/3 of the inputs.
As shown in Figure \ref{fig:shift} and Table \ref{tab:shift}, our approach also performs reliably under the distribution shift in input categories, with a slight increase in average accuracy.

Based on the results, we conclude that our proposed approach can well utilize the advantages of online learning, quickly adapt to unseen inputs, and perform robustly against distribution shifts in the data streams.

\begin{table}[t]
    \centering
    \resizebox{\columnwidth}{!}{%
    \begin{tabular}{lcc}
    \toprule
    \textbf{} & \textbf{GPT-3.5 Turbo} &  \textbf{Llama 2 70B} \\
    \midrule
    Without Any Distribution Shift & 90.77\% & 90.97\% \\
    \midrule
    With Distribution Shift in Input Length & 90.23\% & 90.64\% \\
    $\Rightarrow$ Difference & -0.54\% &-0.33\% \\
    \midrule
    With Distribution Shift in Input Category & 90.85\% &  91.46\% \\
    $\Rightarrow$ Difference & +0.08\% & +0.49\% \\
    \bottomrule
    
    \end{tabular}
    }
    \caption{Average accuracy of our approach across different cost budgets with distribution shifts in input length or input category, compared to the default setting.}
    \label{tab:shift}
\end{table}

%% file: sections/4_related.tex
\section{Related Work}

\subsection{Knowledge Distillation}
Knowledge distillation, originally conceptualized by \citet{hinton2015distilling}, emerged as a technique to transfer knowledge from a large, complex model (teacher) to a smaller, more efficient one (student), intending to retain performance while reducing computational costs. Notable advancements include the works of \citet{sanh2019distilbert}, who demonstrated the effectiveness of distilling the capabilities of BERT into smaller models, and \citet{gu2023knowledge} who successfully applied distillation to LLMs, achieving comparable performance with significantly reduced model sizes. However, the significant difference in capabilities between the teacher and student model can lead to challenges, particularly when dealing with complex queries that require advanced reasoning or involve intricate subject matter \cite{cho2019efficacy, zhang2022lifting}, thereby highlighting a performance gap that distillation alone cannot overcome \cite{rawat2021doubt}. 
% These studies underscore the potential of knowledge distillation in NLP, enabling the deployment of powerful language models in resource-constrained environments.

\subsection{Boosting LLMs with Small Models}
Recent advancements in combining LLMs with smaller models have suggested promising avenues for improving model performance and efficiency. For classification tasks, SuperICL improves LLMs' performance by utilizing the predictions of smaller models to enrich the prompting context \cite{xu2023small}. For generative tasks, speculative decoding techniques \cite{leviathan2023fast, miao2023specinfer, liu2023online} are proposed to accelerate LLM inference by using smaller language models to predict token sequences. Our work differs from these works in that small models are separately trained as plugins, and the performance of LLMs is prioritized. Instead, we focus on the overall performance of the entire cascade.

\subsection{Model Cascade}
By orchestrating multiple models with varying complexities, model cascade has been widely adopted in both CV and NLP tasks to enhance system efficiency. An early theme in this field is the early exiting from a neural network's intermediate layers \cite{liu2020fastbert, xin2020deebert, schwartz2020right}. These works have later inspired the cascade of complete models \cite{li2020cascadebert, khalili2022babybear}. Among them, \citet{DBLP:conf/emnlp/VarshneyB22} systematically examines the trade-off between accuracy and cost in a cascade of variants of BERT. \citet{frugalgpt} further extends the cascade to cover multiple LLM APIs by incorporating a scoring function. Compared to the previous works, our work expands the model cascade by allowing small models to learn online and improve rather than having fixed capabilities, similar to the recent idea of ``neural caching'' \cite{ramirez2023cache, stogiannidis2023cache}. Moreover, by formulating the model confidence as part of the learning objective, our work eliminates the need to set confidence thresholds manually.

%% file: sections/5_conclu.tex
\section{Conclusion}
In this work, we address the challenge of managing streaming queries with LLMs in a cost-efficient way. We propose an online cascade learning framework that adapts to evolving queries by improving smaller models in the cascade through imitation of LLM behaviors. 
% To incorporate model learning and deferral policy learning together, we formulate online cascade learning as an episodic MDP that considers both prediction loss and model costs as a multi-objective optimization problem. 
Our theoretical analysis provides a no-regret performance guarantee for the algorithm. Extensive experiments confirmed the effectiveness of our approach, showing that it can achieve performance levels comparable to LLMs while significantly reducing inference costs, with potential savings of up to 90\%. 
% This work opens up avenues for cost-effective and flexible deployment of LLMs in real-world applications, where queries are dynamic and diverse, highlighting the importance of online learning and adaptive model cascades in the era of large language models.

% \section{Limitations and Future Work}
% Despite the promising experimental results, the current implementation of online cascade learning still has several limitations. One notable limitation is that we obtain LLM predictions under a zero-shot setting without any in-context learning. This can significantly impact the performance of LLMs, especially in complex benchmarks where contextual understanding is crucial. Moreover, the absence of in-context learning may introduce noise into the imitation learning process, potentially affecting the overall effectiveness of our approach.

% Currently, our work primarily focuses on applying the cascade framework to classification tasks. Exploring the extension of the cascade paradigm to generative tasks represents an exciting avenue for future research, potentially revolutionizing the way we approach a wide range of natural language processing applications.

\clearpage

%% file: sections/a_appendix.tex
\section{Detailed Proofs and Theoretical Analysis}
This appendix section provides detailed proofs and explanations for the theoretical analysis presented in the main paper Section \ref{sec:theory}. 
\subsection*{Preliminaries}
To ensure clarity, we begin by defining the concept of regret in online learning.  
\begin{definition}
\label{def:inj}
Given an online learning algorithm operating over time steps $1, ..., T$, the regret $\gamma$ is defined as the difference between the total loss incurred by the algorithm and the loss of the best fixed policy in hindsight. For online cascade learning, regret $\gamma$ is formally expressed as:
\begin{align}
\gamma&= J(\pi, T) - \min\limits_{\pi \in \Pi} J(\pi, T)\\
&=\sum^{T}_{t=1} \sum^{N}_{i=1} p^{s_{t,i}}_{\pi} C_\pi(s_{t,i}) - \min\limits_{\pi\in\Pi}\sum^{T}_{t=1} \sum^{N}_{i=1} p^{s_{t,i}}_{\pi}C_\pi(s_{t,i})
\end{align}
As the value of $t$ is clear from the context, we abbreviate states $s_{t,i} = \langle x_t, i\rangle$ by $s_i$. We denote the best fixed policy as $\pi^* = \argmin_{\pi \in \Pi}J(\pi, T)$ and $\pi (s_{t,i}, \texttt{defer})$ as $p(\pi, s_i)'$. Furthermore, since the prediction loss is computed as the aggregation over a class probability distribution, we simplify $\sum_{y \in Y} \pi (s_{t,i}, y)  \cdot \mathcal{L}(y | y_t)$ as $\mathcal{L}(a_i|y_t)$ by using $a_i$ to represent the output probability vector, which leads to the following representation of regret:
\begin{align}
\gamma&=\sum^{T}_{t=1} \sum^{N}_{i=1} p^{s_i}_{\pi} C_\pi(s_i) - \sum^{T}_{t=1} \sum^{N}_{i=1} p^{s_i}_{\pi^*}C_{\pi^*}(s_i) \label{eq:regret_short}\\
&= \sum^{T}_{t=1} \sum^{N}_{i=1} \prod^{i-1}_{j=1} p(\pi, s_j)' \cdot \Big( (1-p(\pi, s_i)') \cdot \mathcal{L}(a_i|y_t) + p(\pi, s_i)' \cdot \mu c_{i+1} \Big) \nonumber \\ &\ \ - \sum^{T}_{t=1} \sum^{N}_{i=1} \prod^{i-1}_{j=1} p(\pi^*, s_j)' \cdot \Big( (1-p(\pi^*, s_i)') \cdot \mathcal{L}(a^*_i|y_t) + p(\pi^*, s_i)' \cdot \mu c_{i+1} \Big) .
\label{eq:regret}
\end{align}
\end{definition}

\subsection*{Online Ensemble Learning Analysis}

To contextualize the no-regret analysis of our proposed online cascade learning algorithm, we first analyze a simplified \textit{online ensemble learning} algorithm under the same stream processing setting that comprises the linear combination of a series of classification models $m_1, ..., m_N$, each with a static operating probability  $\sum_{i=1}^{N}w_i=1$. Let us denote the model parameters of $m_i$ at time $t$ by $m_i^t$. Assuming a convex, differentiable cost function $c^t$ for all $t$ that can evaluate $m_i^t$, and bounded, closed, nonempty model spaces $||M_i||$ for all $m_i$, we define $||x||=\sqrt{x\cdot x}$ and $d(x,y)=||x-y||$ to establish:
\begin{align}
||M|| &= \max\limits_{x,y \in M_i, i\in[1,N-1]}d(x,y) \nonumber \\
||\nabla c|| &= \underset{m_i\in M_i, i\in[1,N-1], t\in[1,T]}{\max} ||\nabla c^t(m_i)||. \nonumber
\end{align}
\oel*
% \begin{theorem}
% With online gradient descent, if learning rate $\eta_t = t^{-1/2}$, the total regret $\gamma$ of the online ensemble learning algorithm is:    
% \begin{align}
% \label{eq:ensemble_regret}
% \gamma&=\sum^T_{t=1}\sum^N_{i=1}w_i \cdot c^t(m^t_i) - \min\limits_{m_i\in M^i}\sum^T_{t=1}\sum^N_{i=1}w_i\cdot c^t(m^t_i)\\
% &\leq \frac{||M||^2\sqrt{T}}{2} + (\sqrt{T}-\frac{1}{2}) ||\nabla c||^2.
% \end{align}
% Therefore, $\lim_{T\rightarrow \infty} \gamma/T \leq 0$.
% \label{proof:ensemble}
% \end{theorem} 
\begin{proof} 
\label{app:ensemble}
First, we show that, without loss of generality, for all $t$ there exists a $g^t_i \in \mathbb{R}^n$ such that for all models $m_i \in M_i$, $c^t(m_i) = g^t_i \cdot m_i$. 

By defining $g^t_i = \nabla c^t(m_i)$, because $c^t(m_i)$ is convex, for all $m_i \in M_i$:
\begin{equation}
    c^t(m_i) \geq (\nabla c^t (m^t_i)) \cdot (m_i - m^t_i ) + c^t(m^t_i).
\end{equation}
Set $m^*_i$ to be the best-fixed model in hindsight. Then, because $m^*_i \in M_i: c^t(m^*_i) \geq g^t \cdot (m^*_i - m^t ) + c^t(m^t_i)$. Thus,
\begin{align}
    c^t(m^t_i) - c^t(m^*_i) &\leq c^t(m^t_i) - \Big( g^t_i\cdot (m^*_i - m^t_i) + c^t(m^t_i)\Big) \\
    &\leq g^t_i m^t_i - g^t_i m^*_i
\end{align}
Thus we show that for all $m_i$, $g^t_i m^t_i - g^t_i m^*_i$ is the upper bound of $c^t(m^t_i) - c^t(m^*_i)$: 
\begin{align}
\gamma &=\sum^T_{t=1}\sum^N_{i=1}w_i \cdot c^t(m^t_i) - \min\limits_{m_i\in M^i}\sum^T_{t=1}\sum^N_{i=1}w_i\cdot c^t(m^t_i) \\
&=\sum^T_{t=1}\sum^N_{i=1}w_i\Bigl(c^t(m^t_i) - c^t(m^*_i)\Bigr) \\
&\leq \sum^T_{t=1}\sum^N_{i=1}w_i\Bigl(g^t m^t_i - g^t m_i^*\Bigr).
\label{eq:ensemble_regret_bound}
\end{align}

We define for all $t$, $\hat{m_i}^{t+1} = m_i^t - \eta_t g^t_i$ (gradient descent). Note that $m_i^{t+1} = P(\hat{m_i}^{t+1}) = \argmin\limits_{m_i\in M_i} || m_i - \hat{m_i}^{t+1}|| $ (greedy projection). We will attempt to bound the regret of not playing action $m^*_i$ on round $t$:
\begin{align}
    \hat{m_i}^{t+1} - m^*_i &=  m^t_i - \eta_t g^t_i - m^*_i \\
    (\hat{m_i}^{t+1} - m_i^*)^2 &= (m_i^t - m_i^*)^2 - 2\eta_t (m_i^t - m_i^*) \cdot g^t_i + \eta_t^2 ||g^t_i||^2
\end{align}

Since by definition, for all $\hat{m_i} \in \mathbb{R}^n$, for all $m_i \in M_i$, $(\hat{m_i} - m_i)^2 \geq (P(\hat{m_i}) - m_i)^2$. Also, $||g^t_i|| \leq ||\nabla c||$. So
\begin{align}
    (m^{t+1}_i-m^*_i)^2 &\leq (\hat{m_i}^{t+1} - m^*_i)^2 \\
    &\leq (m^t_i - m^*_i)^2 - 2\eta_t (m^t_i - m^*_i) \cdot g^t_i + \eta_t^2 ||\nabla c||^2,
\end{align} which can be converted into:
\begin{equation}
\label{eq:ensemble_proof}
    g^t_i m^t_i - g^t_i m^*_i \leq \frac{1}{2\eta_t}\Bigl[ (m^t_i - m^*_i)^2 - (m^{t+1}_i-m^*_i)^2\Bigr] + \frac{\eta_t}{2}||\nabla c||^2.
\end{equation}
By summing Equation (\ref{eq:ensemble_regret_bound}) and Equation (\ref{eq:ensemble_proof}) we get:
\begin{align}
    \gamma &\leq \sum^T_{t=1}\sum^N_{i=1}w_i\Bigl(g^t m^t_i - g^t m_i^*) \\
    &\leq \sum^T_{t=1}\sum^N_{i=1}w_i \biggr[ \frac{1}{2\eta_t}\Bigl[ (m^t_i - m^*_i)^2 - (m^{t+1}_i-m^*_i)^2\Bigr] + \frac{\eta_t}{2}||\nabla c||^2  \biggl] \\
    &\leq \sum^N_{i=1} \frac{w_i}{2\eta_t} (m^1_i-m^*_i)^2  - \sum^N_{i=1} \frac{w_i }{2\eta_t} (m^{T+1}_i-m^*_i)^2  \\
    &+ \frac{1}{2} \sum^T_{t=2}\sum^N_{i=1}w_i  (\frac{1}{\eta_t} - \frac{1}{\eta_{t-1}})(m^t_i-m^*_i)^2 + \frac{||\nabla c||^2}{2} \sum^T_{t=1} \sum^N_{i=1}w_i \cdot \eta_t \\
    &\leq \sum^N_{i=1}w_i ||M||^2 \Bigl(\frac{1}{2\eta_1} + \frac{1}{2}\sum^T_{t=2}(\frac{1}{\eta_t} - \frac{1}{\eta_{t-1}})\Bigr) + \frac{||\nabla c||^2}{2} \sum^T_{t=1} \sum^N_{i=1}w_i \cdot \eta_t
\end{align}
Since $ \sum^N_{i=1}w_i = 1$,
\begin{align}
    \gamma &\leq ||M||^2 \Bigl(\frac{1}{2\eta_1} + \frac{1}{2}\sum^T_{t=2}(\frac{1}{\eta_t} - \frac{1}{\eta_{t-1}})\Bigr) + \frac{||\nabla c||^2}{2} \sum^T_{t=1} \eta_t \\
    &\leq ||M||^2 \frac{1}{2\eta_T} + \frac{||\nabla c||^2}{2} \sum^T_{t=1} \eta_t 
\label{eq:ensemble_last}
\end{align}
Now, if we define $\eta_t = \frac{1}{\sqrt{t}}$, then
\begin{align}
    \sum^T_{t=1} \eta_t &= \sum^T_{t=1} \frac{1}{\sqrt{t}} \\
    &\leq 1 + \int^T_{t=1} \frac{dt}{\sqrt{t}} \\
    &\leq 1 + [2\sqrt{t}]^T_{1} \\
    &\leq 2\sqrt{T} - 1
\end{align}
Plugging this into Equation (\ref{eq:ensemble_last}) yields
\begin{align}
    \gamma &\leq \frac{||M||^2\sqrt{T}}{2} + \frac{||\nabla c||^2}{2}(2\sqrt{T}-1) \\
    &\leq \frac{||M||^2\sqrt{T}}{2} + (\sqrt{T}-\frac{1}{2}) ||\nabla c||^2
\end{align} \\
Therefore, when the number of iterations $T$ approaches infinity, the average regret $\lim_{T \rightarrow \infty} \gamma/T \leq 0$.
\end{proof}

\subsection*{Extension to Online Cascade Learning}
The above theorem establishes that the online ensemble learning algorithm achieves no regret as $T \rightarrow \infty$. Transitioning to online cascade learning, we extend this analysis by substituting the fixed model probabilities $w_i$ in Equation (\ref{eq:ensemble_regret}) with dynamic probabilities influenced by preceding model actions (\ie $p^{s_i}_{\pi}=\prod^{i-1}_{j=1} p(\pi, s_j)'$) in Equation (\ref{eq:regret}). Thus, we now analyze the convergence of the deferral policies $p^{s_i}_{\pi}$.

\begin{lemma}
\label{proof:lemma}
For online cascade learning, for all $i\in [1,N]$, $p^{s_i}_{\pi}-p^{s_i}_{\pi^*} \xrightarrow[]{T\rightarrow\infty} 0$.
\end{lemma}
\begin{proof}
\label{app:lemma}
We begin by revisiting Equation (\ref{eq:regret}), transforming its first term to facilitate our analysis:
\begin{align}
&\ \ \ \ \ \sum^{T}_{t=1} \sum^{N}_{i=1} \prod^{i-1}_{j=1} p(\pi, s_j)' \cdot \Big( (1-p(\pi, s_i)') \cdot \mathcal{L}(a_i|y_t) + p(\pi, s_i)' \cdot \mu c_{i+1} \Big) \\ 
&=  \sum^{T}_{t=1} \sum^{N}_{i=1} \prod^{i-1}_{j=1}p(\pi, s_j)' \cdot \Big(  p(\pi, s_i)' \big(\mu c_{i+1} - \mathcal{L}(a_i|y_t)) + \mathcal{L}(a_i|y_t)\Big).
\label{transformed_regret}
\end{align}
Given that $p(\pi, s_i)'$ lies within the range $(0, 1)$ for all $t$ and $i$, the coefficient preceding $\mathcal{L}(a_i|y_t)$, namely $\prod^{i-1}_{j=1} p(\pi, s_j)' (1-p(\pi, s_i)')$, also falls within $(0, 1)$. Assuming $\mathcal{L}$ is convex, and following the convergence arguments in \citet{li2019convergence}, we establish that as $T \rightarrow \infty$, $\mathcal{L}(a_i|y_t)$ converges to a minimal loss value $\epsilon_i$ for each model.

Integrating this convergence into our transformed regret expression of Equation (\ref{transformed_regret}), we arrive at:
\begin{equation}
    \sum^{T}_{t=1} \sum^{N}_{i=1} \prod^{i-1}_{j=1}p(\pi, s_j)' \cdot \Big(  p(\pi, s_i)' \big(\mu c_{i+1} - \epsilon_i) + \epsilon_i\Big).
    \label{objective}
\end{equation}
As discussed in \citet{confidence} Proposition 3.1, to minimize the aggregated costs, the optimal deferral rule $p(\pi^*, s_i)'$ should be:
\begin{equation*}
        p(\pi^*, s_i)'=\begin{cases}
    0& \text{if $\mu c_{i+1} - \epsilon_i > 0$},\\
    1& \text{otherwise.}
  \end{cases}
\end{equation*}

With gradient descent, as $T\rightarrow \infty$, we can show that for all $i$, $p(\pi, s_i)'$ gradually approaches $p(\pi^*, s_i)'=\mathbbm{1}[\mu c_{i+1} - \epsilon_i \leq 0 ]$ by pointwise convergence. Therefore, in the base case of $N=2$, we have $p^{s_1}_{\pi}-p^{s_1}_{\pi^*} = p(\pi, s_1)' - p(\pi^*, s_1)' = 0, p^{s_2}_{\pi}-p^{s_2}_{\pi^*} = 1-1=0$.

Assuming the lemma holds for $N=k$, \ie $p^{s_i}_{\pi} - p^{s_i}_{\pi^*} = 0$ for all $i \in [1, k]$, we extend this to $N=k+1$:
\begin{align}
    &p(\pi, s_{k+1})' - p(\pi^*, s_{k+1})' \\ 
    &= \prod^{k}_{j=1} p(\pi, s_j)' - \prod^{k}_{j=1} p(\pi^*, s_j)' \\
    &= p^{s_k}_{\pi} \cdot p(\pi, s_k)' - p^{s_k}_{\pi^*} \cdot p(\pi^*, s_k)' \\
    &= p^{s_k}_{\pi^*} \Big( p(\pi, s_k)' - p(\pi^*, s_k)' \Big) = 0
\end{align}

Hence by mathematical induction, for online cascade learning, when $T \rightarrow \infty$, for all $i\in [1,N]$, $p^{s_i}_{\pi}-p^{s_i}_{\pi^*} = 0$.
\end{proof}
With the conclusions of Theorem \ref{proof:ensemble} and Lemma \ref{proof:lemma}, we can now analyze the performance of our proposed online cascade learning algorithm.
\ocl*
% \begin{theorem}
% For online cascade learning, if learning rate $\eta_t = t^{-1/2}$, the total regret of the algorithm is $o(T)$. Therefore, $\lim_{T\rightarrow\infty} \gamma/T \leq 0$, \ie the average regret approaches $0$ when $T$ approaches $\infty$.
% \end{theorem}
\begin{proof}
\label{app:cascade}
Leveraging the findings of Lemma \ref{proof:lemma}, assuming $\mathcal{L}$ is a convex function and $T\rightarrow \infty$, we can now reformulate the regret expression in Equation (\ref{eq:regret_short}) as:

\begin{align}
\gamma&=\sum^{T}_{t=1} \sum^{N}_{i=1} p^{s_i}_{\pi} C_\pi(s_i) - \sum^{T}_{t=1} \sum^{N}_{i=1} p^{s_i}_{\pi^*}C_{\pi^*}(s_i) \\
&=\sum^{T}_{t=1} \sum^{N}_{i=1} p^{s_i}_{\pi^*} C_\pi(s_i) - \sum^{T}_{t=1} \sum^{N}_{i=1} p^{s_i}_{\pi^*}C_{\pi^*}(s_i)
\label{eq:regret_final}
\end{align}

The transformation of the regret expression to Equation (\ref{eq:regret_final}) is significant as it bridges our understanding of the regret in online cascade learning with the established results from online ensemble learning. Specifically, by treating $p^{s_i}_{\pi^*}$ as analogous to the fixed model probabilities $w_i$ and $C_{\pi}(s_i)$ as equivalent to the costs $c^t(m_i^t)$ in the ensemble learning context, we establish a parallel between the two regret formulations.

Therefore, following the conclusion of Theorem \ref{proof:ensemble}, we can infer that, with a learning rate $\eta_t = t^{-1/2}$, the online cascade learning algorithm is guaranteed can achieve no regret, \ie $\lim_{T\rightarrow\infty} \gamma/T \leq 0$. 

\end{proof}

%% file: sections/b_appendix.tex
\section{Detailed Experimental Setups}
\subsection{A Simple Prefill Experiment} \label{sec:prefill}

To determine time required to process a document-plus-prompt in a realistic scenario, we performed the following, simple experiment. Using the 65B parameter LLaMA model \cite{touvron2023llama1} and PyTorch 2.1.2, on an Amazon Web Services `\texttt{m6in.16xlarge}' machine with eight, A100 GPUs, we prepare a sequence of 10 prompts each consisting of 8192 tokens, and perform ``first token'' inference on each in sequence. That is, we process the prompt and obtain the first output token, and record the total time taken, which is 36.2 seconds, for an average of 3.6 seconds per prompt. 

The reason we perform first-token inference is that relative to the first output token---which requires a quadratic all-to-all attention computation---subsequent tokens are relatively costless, taking a fraction of a second. For our application, where the full response sequence is expected to be short, first token inference is by far the most costly step.  Note that in our experiment, inference is performed separately on each prompt, rather than as a batch. This is necessary to avoid out-of-memory errors (there is not enough GPU memory to process more than one prompt at a time, due to the memory requirements of the all-to-all attention computation).

\subsection{LLM Prompts}
\label{sec:prompt}
\textbf{IMDB} \\
\textit{System Prompt}: You are a helpful, respectful and honest assistant. The user has given you a movie review to help them make their decision. You should read the review and tell the user whether the review overall shows a positive or negative sentiment towards the movie. Return the answer in one word. \\
\textit{User Prompt}: Here is the movie review: \texttt{\{REVIEW\}} \textbackslash\textbackslash Tell me whether the above review overall shows a positive or negative sentiment towards the movie. Return the answer in one word.

\textbf{HateSpeech} \\
\textit{System Prompt}: You are given a post from an online forum and you need to check whether the post contains any hate speech. Return your answer in one word (yes or no) without any explanations.  \\
\textit{User Prompt}: Post: \texttt{\{POST\}}

\begin{figure*}[t!]
    \centering
\includegraphics[height=130pt]{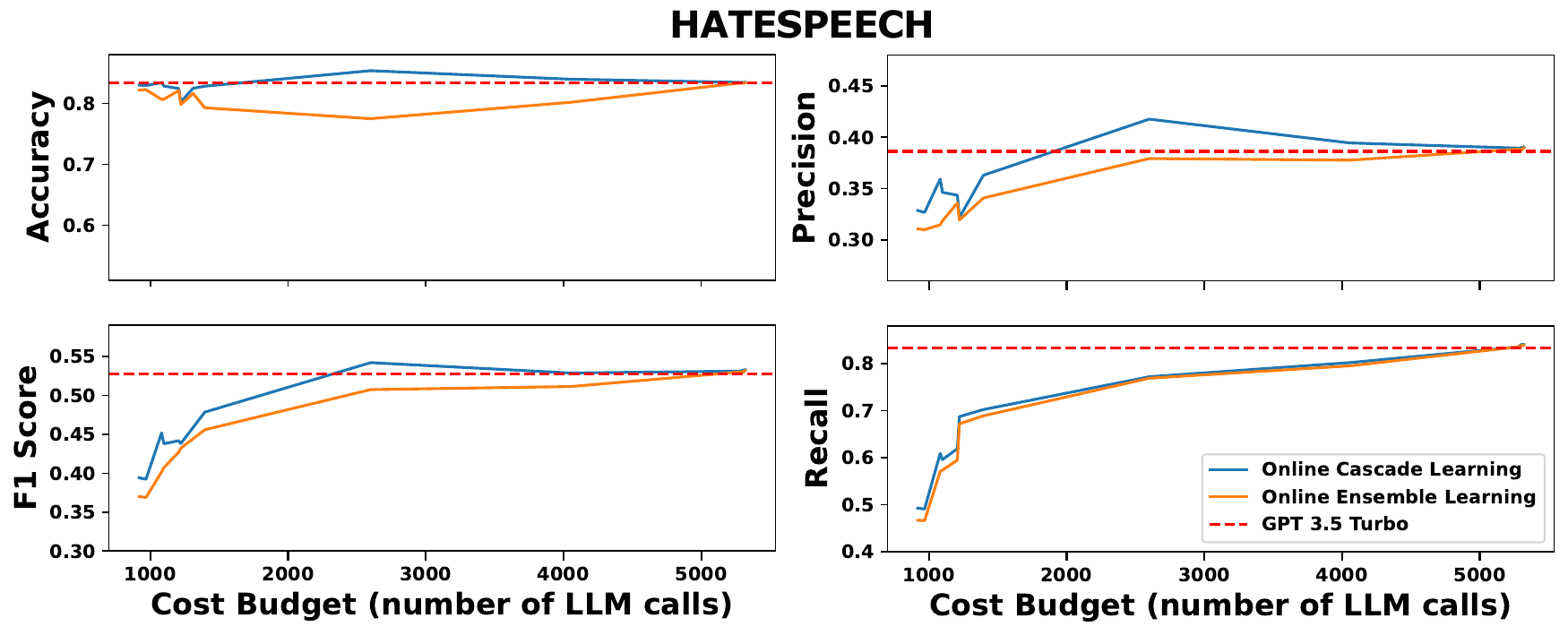}
\includegraphics[height=130pt]{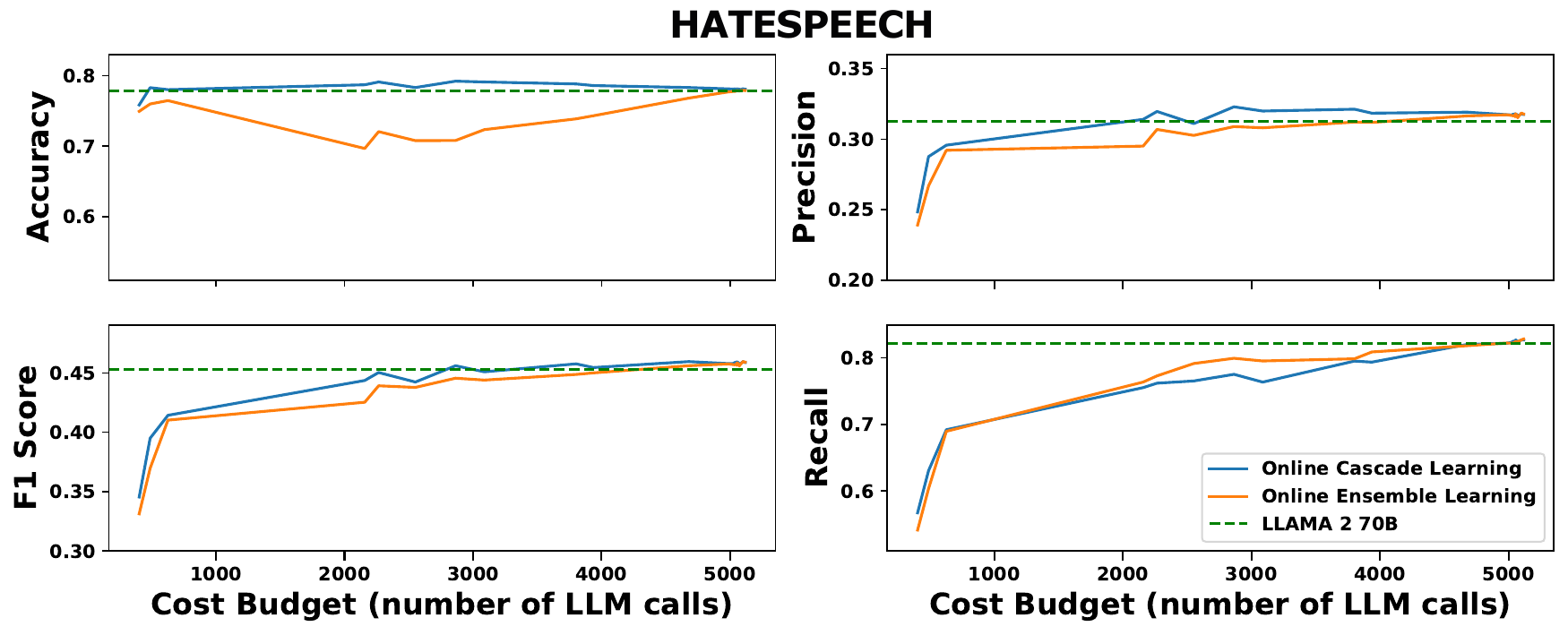}
    \caption{A more detailed cost-performance trade-off plot with accuracy, F1-score, recall, and precision curves, respectively using GPT-3.5 Turbo and Llama 2 70B Chat as the LLM.}
    \label{fig:f1_gpt}
\end{figure*}

\begin{figure*}[t]
    \centering
{\includegraphics[height=95pt]{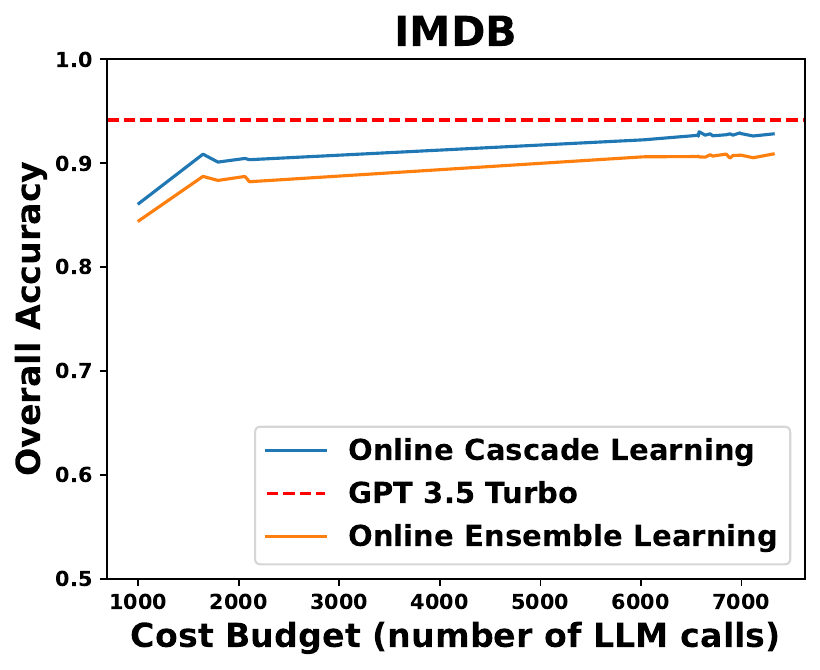}}
{\includegraphics[height=95pt]{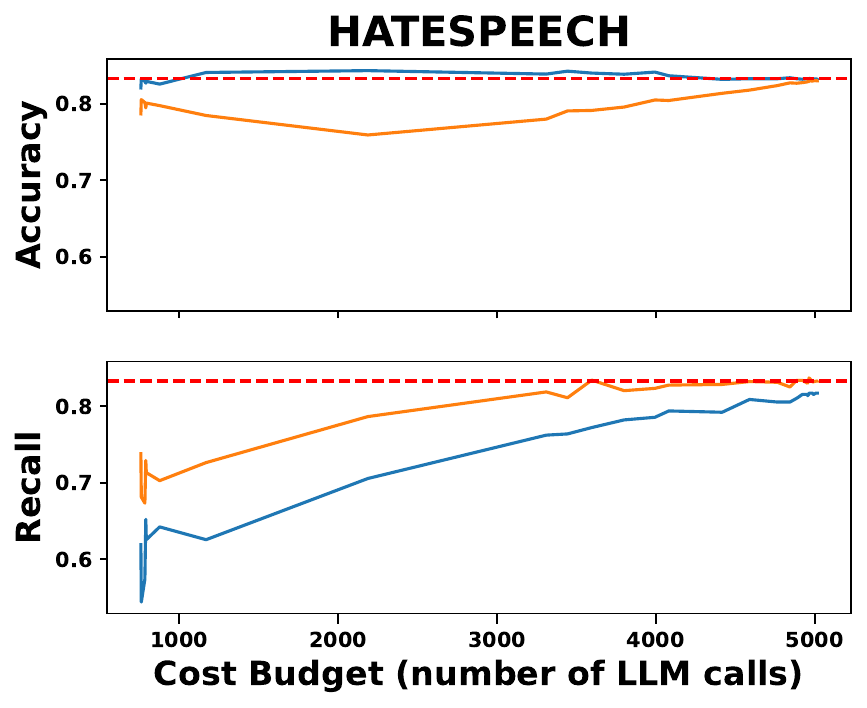}}
{\includegraphics[height=95pt]{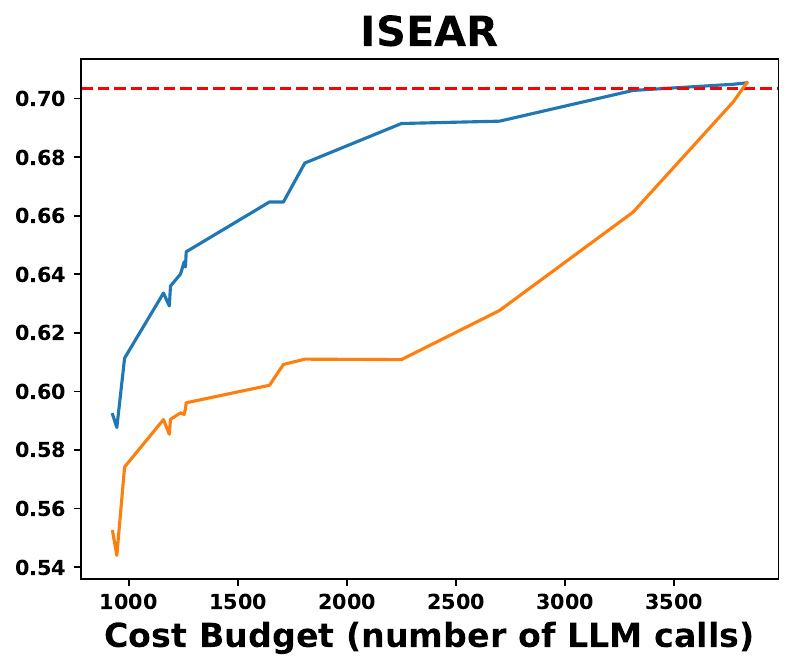}} 
{\includegraphics[height=95pt]{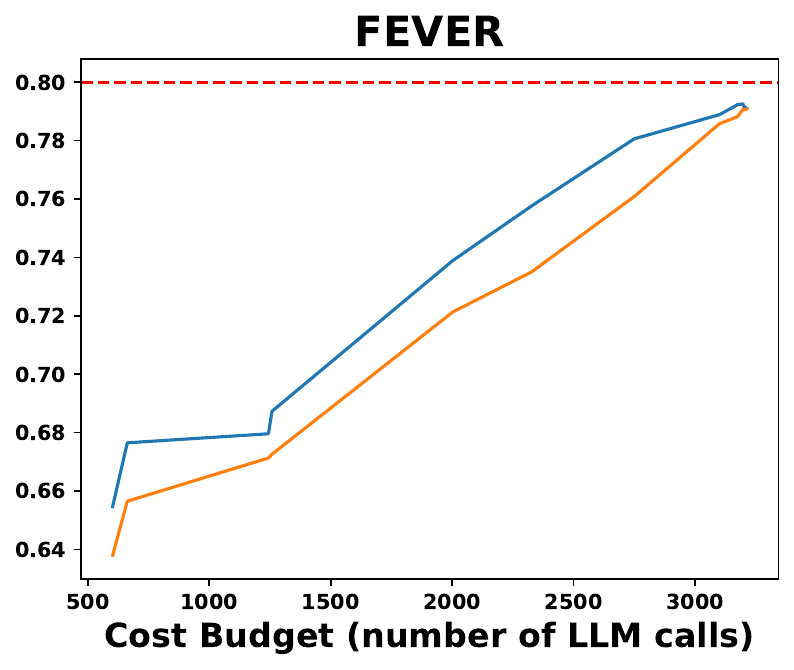}} 

{\includegraphics[height=95pt]{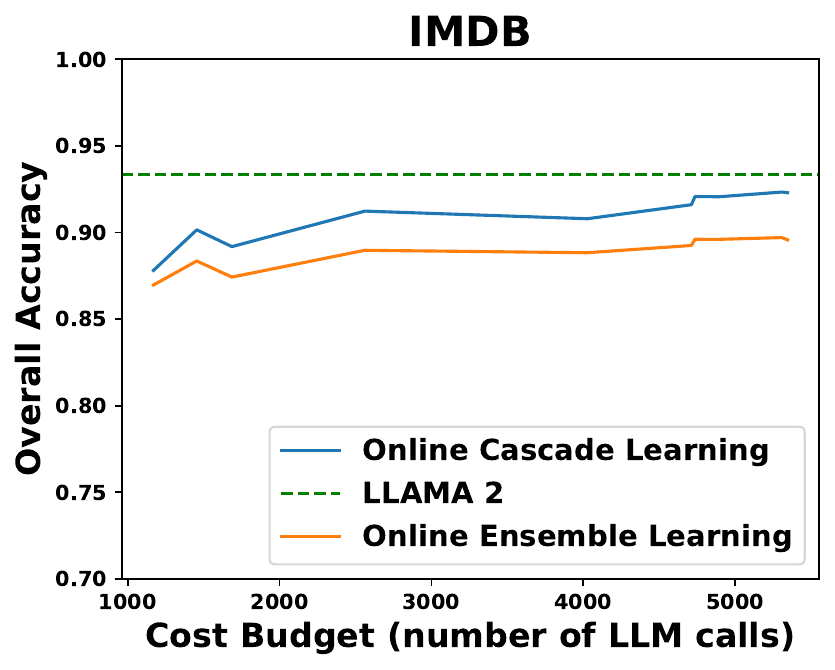}}
{\includegraphics[height=95pt]{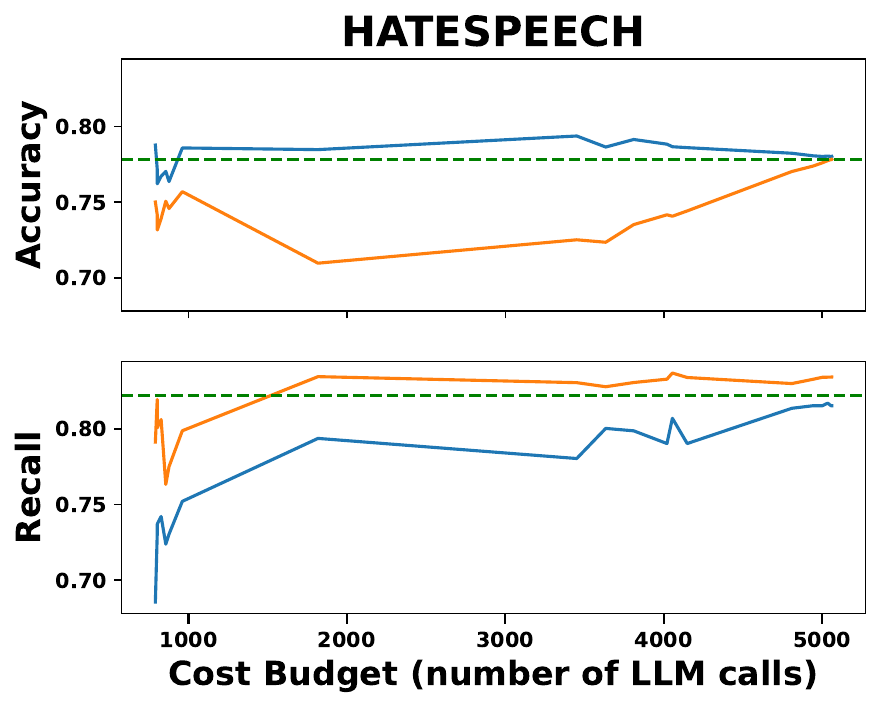}}
{\includegraphics[height=95pt]{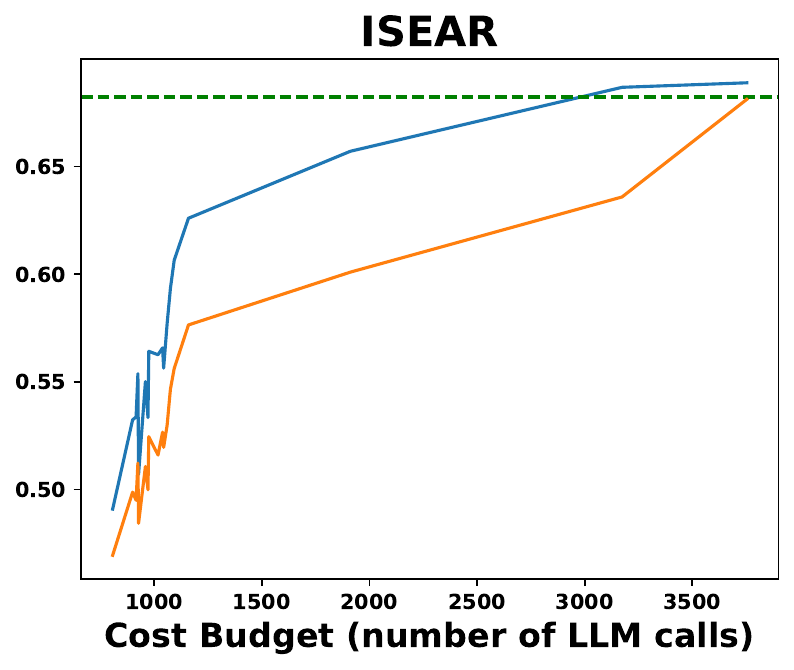}} 
{\includegraphics[height=95pt]{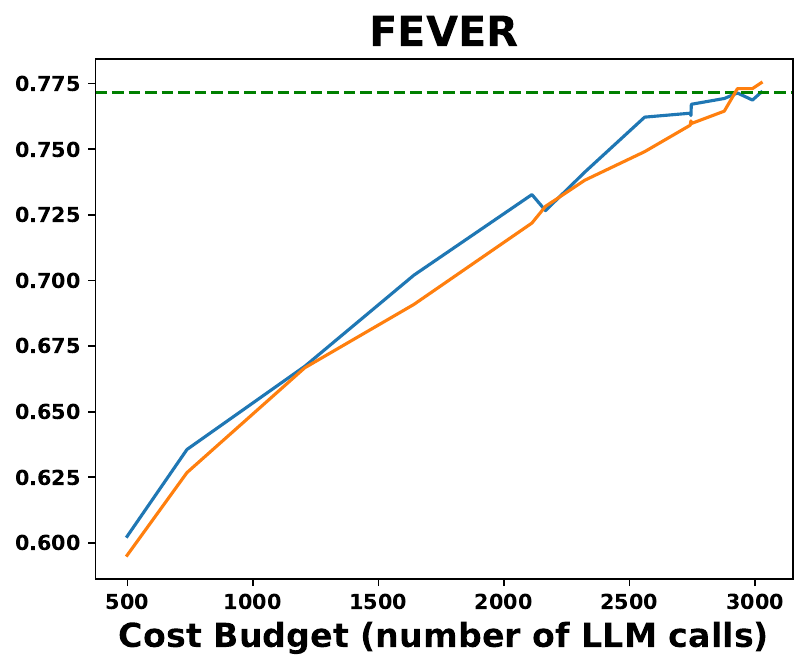}} 

    \caption{Accuracy curve (and Recall curve for HateSpeech) with respect to costs, respectively using GPT-3.5 Turbo and Llama 2 70B Chat as the LLM in a larger cascade that also comprises logistic regression, BERT-base, and BERT-large. }
    \label{fig:large_tradeoff_gpt}
\end{figure*}

% \begin{figure*}[t]
%     \centering
% {\includegraphics[height=95pt]{figures/IMDB_large_LLAMA_cascade.pdf}}
% {\includegraphics[height=95pt]{figures/HATESPEECH_large_LLAMA_cascade.pdf}}
% {\includegraphics[height=95pt]{figures/ISEAR_large_LLAMA_cascade.pdf}} 
% {\includegraphics[height=95pt]{figures/FEVER_large_LLAMA_cascade.pdf}} 

%     \caption{Accuracy curve (and Recall curve for HateSpeech) with respect to costs, using Llama 2 70B Chat as the LLM in a larger cascade that also comprises logistic regression, BERT-base, and BERT-large.}
%     \label{fig:large_tradeoff_llama}
% \end{figure*}

\textbf{ISEAR} \\
\textit{System Prompt}: In this task, you will be performing a classification exercise aimed at identifying the underlying emotion conveyed by a given sentence.
The emotions to consider are as follows:\\
Anger: Anger is a strong feeling of displeasure, hostility, or frustration. \\
Joy: Joy is a positive and uplifting emotion characterized by happiness, elation, and a sense of contentment. \\
Sadness: Sadness is a feeling of sorrow, unhappiness, or despondency. \\
Guilt: Guilt is a self-directed emotion that arises from a sense of wrongdoing or moral transgression. \\
Shame: Shame is a powerful emotion associated with feeling embarrassed, humiliated, or unworthy. \\
Fear: Fear is an emotion triggered by a perceived threat or danger. \\
Disgust: Disgust is an aversive emotion linked to feelings of revulsion, repulsion, or strong distaste. It arises in response to things that are offensive or unpleasant.\\
Your task is to analyze each sentence provided and categorize it into one of these emotions based on the dominant feeling conveyed by the text.\\
This classification will require an understanding of the nuances of human emotions and the context in which the sentences are presented.
Remember, you have to classify the sentences using only anger, joy, sadness, guilt, shame, fear or disgust. Please respond with only the word and nothing else. \\
\textit{User Prompt}: \texttt{\{SENTENCE\}} \textbackslash\textbackslash Classify the emotion of this hypothetical sentence. Respond in exactly one word in all lowercase with a response in the exact format requested by the user. Do not acknowledge my request with ``sure" or in any other way besides going straight to the answer. Only answer in exactly one word.

\textbf{FEVER} \\
\textit{System Prompt}: You are a helpful, respectful and honest assistant. This is a fact-checking task. Use your knowledge to determine whether a given claim is true or false. Answer only in ``true" or ``false" without providing any explanations.   \\
\textit{User Prompt}: In June 2017, the following claim was made:  \texttt{\{CLAIM\}}.

\subsection{Experimental Configurations}
The experiments involved querying Llama 2 70B Chat utilized a single machine equipped with 8 NVIDIA A40 GPUs, each with 48GB of memory, running CUDA 12.0. All the other experiments were conducted on a machine with 4 NVIDIA Quadro RTX 8000 GPUs (48GB memory each) on CUDA 12.2.

The detailed hyperparameter settings for online cascade learning are listed in Table \ref{app:tab1} and \ref{app:tab2}. We tuned the hyperparameters using a grid search method on a separate validation set, which is a standard practice to avoid overfitting. In particular, we used the training set prepared for the knowledge distillation (as mentioned in Section \ref{sec:baseline}) as our online cascade learning method’s validation set.

Specifically, for the hyperparameters $\beta$ and $\mu$, we observed that our experimental results are notably robust to variations in $\beta$. Regarding $\mu$, we tuned it specifically in the context of different cost budgets, which was essential for plotting Figure \ref{fig:tradeoff_gpt} and Figure \ref{fig:tradeoff_llama} in our paper. By adjusting $\mu$, we were able to effectively manage the cost budgets to evaluate our method’s performance on cost-accuracy trade-offs, which is a critical aspect of our research objective.

Note that the ``learning rate'' in the table refers to the learning rates of the MLPs (in Section \ref{sec:conf}: Confidence Calibration), not the models'. We used a consistent configuration for both online cascade learning and the distillation baselines by setting BERT-base's batch size to 8, the learning rate to 0.00001, and the number of epochs to 5. 

\begin{table}[t]
\centering
\caption{Hyperparameter settings for online cascade learning experiments with GPT-3.5 Turbo as the LLM.}
\label{tab:hyperparameters1}
\begin{tabular}{lcccccc}
\hline
 & Model Cost & Cache Size & Batch Size & Learning Rate & Decaying Factor & Calibration Factor \\
\hline
\multicolumn{7}{l}{\textbf{IMDB, Small Cascade}} \\
LR & 1 & 8 & 8 & 0.0007 & 0.97 & 0.4 \\
BERT-base & 1182 & 16 & 8 & 0.0007 & 0.95 & 0.3 \\
\hline
\multicolumn{7}{l}{\textbf{IMDB, Large Cascade}} \\
LR & 1 & 8 & 8 & 0.0007 & 0.99 & 0.45 \\
BERT-base & 3 & 16 & 8 & 0.0007 & 0.97 & 0.4 \\
BERT-large & 1182 & 32 & 16 & 0.0007 & 0.95 & 0.4 \\
\hline
\multicolumn{7}{l}{\textbf{HateSpeech, Small Cascade}} \\
LR & 1 & 8 & 8 & 0.001 & 0.97 & 0.4 \\
BERT-base & 1182 & 16 & 8 & 0.0007 & 0.9 & 0.4 \\
\hline
\multicolumn{7}{l}{\textbf{HateSpeech, Large Cascade}} \\
LR & 1 & 8 & 8 & 0.001 & 0.99 & 0.45 \\
BERT-base & 3 & 16 & 8 & 0.0007 & 0.97 & 0.45 \\
BERT-large & 1182 & 32 & 16 & 0.0007 & 0.95 & 0.45 \\
\hline
\multicolumn{7}{l}{\textbf{ISEAR, Small Cascade}} \\
LR & 1 & 8 & 8 & 0.0007 & 0.8 & 0.15 \\
BERT-base & 1182 & 16 & 8 & 0.0007 & 0.9 & 0.45 \\
\hline
\multicolumn{7}{l}{\textbf{ISEAR, Large Cascade}} \\
LR & 1 & 8 & 8 & 0.0007 & 0.99 & 0.4 \\
BERT-base & 3 & 16 & 8 & 0.0007 & 0.97 & 0.35 \\
BERT-large & 1182 & 32 & 16 & 0.0007 & 0.95 & 0.3 \\
\hline
\multicolumn{7}{l}{\textbf{FEVER, Small Cascade}} \\
LR & 1 & 8 & 8 & 0.0007 & 0.97 & 0.4 \\
BERT-base & 1182 & 16 & 8 & 0.0007 & 0.95 & 0.3 \\
\hline
\multicolumn{7}{l}{\textbf{FEVER, Large Cascade}} \\
LR & 1 & 8 & 8 & 0.0007 & 0.97 & 0.4 \\
BERT-base & 3 & 16 & 8 & 0.001 & 0.95 & 0.4 \\
BERT-large & 1182 & 32 & 16 & 0.0001 & 0.93 & 0.4 \\
\hline
\end{tabular}
\label{app:tab1}
\end{table}

\begin{table}[h!]
\centering
\caption{Hyperparameter settings for online cascade learning experiments with Llama 2 70B Chat as the LLM.}
\label{tab:hyperparameters2}
\begin{tabular}{lcccccc}
\hline
 & Model Cost & Cache Size & Batch Size & Learning Rate & Decaying Factor & Calibration Factor \\
\hline
\multicolumn{7}{l}{\textbf{IMDB, Small Cascade}} \\
LR & 1 & 8 & 8 & 0.0007 & 0.97 & 0.4 \\
BERT-base & 636 & 16 & 8 & 0.0007 & 0.95 & 0.3 \\
\hline
\multicolumn{7}{l}{\textbf{IMDB, Large Cascade}} \\
LR & 1 & 8 & 8 & 0.0007 & 0.99 & 0.45 \\
BERT-base & 3 & 16 & 8 & 0.0007 & 0.97 & 0.4 \\
BERT-large & 636 & 32 & 16 & 0.0007 & 0.95 & 0.4 \\
\hline
\multicolumn{7}{l}{\textbf{HateSpeech, Small Cascade}} \\
LR & 1 & 8 & 8 & 0.001 & 0.97 & 0.4 \\
BERT-base & 636 & 16 & 8 & 0.0007 & 0.9 & 0.4 \\
\hline
\multicolumn{7}{l}{\textbf{HateSpeech, Large Cascade}} \\
LR & 1 & 8 & 8 & 0.001 & 0.99 & 0.45 \\
BERT-base & 3 & 16 & 8 & 0.0007 & 0.97 & 0.45 \\
BERT-large & 636 & 32 & 16 & 0.0007 & 0.95 & 0.45 \\
\hline
\multicolumn{7}{l}{\textbf{ISEAR, Small Cascade}} \\
LR & 1 & 8 & 8 & 0.0007 & 0.8 & 0.15 \\
BERT-base & 636 & 16 & 8 & 0.0007 & 0.9 & 0.45 \\
\hline
\multicolumn{7}{l}{\textbf{ISEAR, Large Cascade}} \\
LR & 1 & 8 & 8 & 0.0007 & 0.99 & 0.4 \\
BERT-base & 3 & 16 & 8 & 0.0007 & 0.97 & 0.35 \\
BERT-large & 636 & 32 & 16 & 0.0007 & 0.95 & 0.3 \\
\hline
\multicolumn{7}{l}{\textbf{FEVER, Small Cascade}} \\
LR & 1 & 8 & 8 & 0.0007 & 0.97 & 0.4 \\
BERT-base & 636 & 16 & 8 & 0.0007 & 0.95 & 0.3 \\
\hline
\multicolumn{7}{l}{\textbf{FEVER, Large Cascade}} \\
LR & 1 & 8 & 8 & 0.0007 & 0.97 & 0.4 \\
BERT-base & 3 & 16 & 8 & 0.001 & 0.95 & 0.4 \\
BERT-large & 636 & 32 & 16 & 0.0001 & 0.93 & 0.4 \\
\hline
\end{tabular}
\label{app:tab2}
\end{table}

\section{Discussion}

\begin{table}[t]
    \centering
    % \resizebox{\columnwidth}{!}{%
    \begin{tabular}{lccc}
    \toprule
    \textbf{Length} & \textbf{Count} &  \textbf{Average Length} & \textbf{GPT-3.5 Turbo Accuracy} \\
    \midrule
    52-664 &4975 &481.92& 95.54\% \\
    664-843 & 5018 & 745.86 & 95.08\% \\
    843-1160 & 5003 & 985.05 & 93.96\% \\
    1160-1852 & 5001 & 1453.49 & 93.74\% \\
    1852-13704 & 5002 & 2953.95 & 92.44\% \\
    \midrule
    Total & 25000 & 1325.07 & 94.15\% \\
    \bottomrule
    
    \end{tabular}
    % }
    \caption{GPT-3.5 Turbo's classification accuracy across different IMDB review lengths. Longer inputs are typically more complex and thus have lower average accuracies. }
    \label{tab:length}
\end{table}

\subsection{Analysis of Training \& Inference Cost Equilibrium}
\label{app:equi}

We conduct a thorough analysis of the training overhead for the models involved in our cascade to quantify the computational costs incurred in our approach. Below are the computation costs for training or inference over one sample (the computational costs of the confidence calibration MLP in Section \ref{sec:conf}, inference: $897\ Flops$, training: $1794\ Flops$, are negligible):

\vspace{-\topsep}
\begin{itemize}
\itemsep-0.2em 
    \item Logistic Regression training: $33.8 \times 10^4 Flops$ (floating-point operations).
    \item Logistic Regression inference: $16.9 \times 10^4 Flops$.
    \item BERT-base training: $18.5 \times 10^7 Flops$.
    \item BERT-base inference: $9.2 \times 10^7 Flops$.
    \item BERT-large training: $55.5\times10^7 Flops$.
    \item BERT-large inference: $27.7\times10^7 Flops$.
\end{itemize}

Comparatively, Llama 2 70B’s inference costs for generating one token is approximately $39.86 \times 10^{15}Flops$ (we have no access to GPT-3.5 Turbo’s running statistics). Even if all the small models are
\textbf{consistently updated per sample}, the per-sample training costs of a large cascade $\underset{\textit{LR}}{33.8\times10^4}+\underset{\textit{BERT-base}}{18.5\times10^7} +\underset{\textit{BERT-large}}{55.5\times10^7} \approx7.4\times10^8 Flops$ is still $5.3\times10^7$ times smaller than the per-sample Llama 2 70B inference costs. 

Therefore, the computational costs regarding the deferral policy's inference and training are minimal when comparing against the enormous LLM inference costs. In the real world, as the smaller models' capabilities grow over time and the need for training decreases, the incurred model training overhead is also negligible compared to the tremendous LLM inference costs saved with our approach.

More formally, we can formulate the theoretical cost equilibrium as follows: 
$$100\% \cdot C = x \cdot M + (1-x) \cdot (M + 2M + C), $$ where the LHS refers to the inference overheads for using the LLM to process all queries, and the RHS comprises the maximum inference costs forusing the small models to handle $x\%$ of the queries (which are not deferred to the LLM), and the LLM inference costs, plus the inference \& training costs for updating the small models when handling the rest $(1-x\%)$ queries. This equation can be further simplified to $$M = \frac{xC}{3-2x},$$ where $C$ represents the LLM inference cost, $x$ is the proportion of queries handled by small models, and $M$ indicates small models’ aggregated inference costs.

For example, assuming $C=39.86\times10^{15} Flops$, $x=0.5$, then $M\approx 9.95\times10^{15} Flops$,typically refers to a total number of parameters around 17.5 Billion. That means, when using Llama 2 70B as the LLM, even if the smaller models can only handle 50\% of the input queries, as long as the smaller models’ total number of parameters does not exceed 17.5B, our approach can still save costs.

\subsection{Potential Overfitting in Larger Cascades}
\label{app:larger}
When scaling up the cascade, the task complexity may influence the appropriate cascade size, as suggested by the performance dynamics observed in Section \ref{sec:larger}. There are several factors that may affect this:

\paragraph{Task Complexity vs. Cascade Size} The HateSpeech dataset, despite its class imbalance, represents arelatively simple binary classification task where hate speech is often identifiable through specific keywords. Our findings, illustrated in Figure \ref{fig:hatespeech_exp}, show that a basic cascade of logistic regression and BERT-base can already handle ~90\% of the queries effectively, matching the performance of more complex models like GPT3.5 Turbo. This indicates that for simpler tasks, adding more models to the cascade might introduce unnecessary noise, complicating the deferral policy learning and thus degrading overall performance.

\paragraph{Capability Gap in Models} The performance impact of adding more models to a cascade also depends onthe capability gap between these models. In our larger cascade setup, we incorporated BERT-large alongside BERT-base. However, both models exhibit similar performance on the HateSpeech dataset, meaning the addition of BERT-large brought minimal benefit. This redundancy shows that the effectiveness of a cascade does not solely rely on adding more models, but rather on ensuring that each model contributes unique capabilities to the given task.

\paragraph{Improved Performance on More Complex Dataset} In contrast, for the more complex ISEAR dataset, the larger cascade that includes BERT-large outperformed the smaller cascade, despite the ISEAR dataset's smaller size (3833) compared to HateSpeech (5352). This outcome supports the idea that as task complexity increases, a larger cascade can still better balance the cost-performance trade-off, leveraging the distinct strengths of each model in the cascade.

\subsection{Challenges in Scaling Up Cascade}
\label{app:llm}

Scaling up the cascade for more complex task processing is technically feasible and compatible with our proposed framework, but can present notable challenges:

\paragraph{Confidence Measurement in Generative Tasks} For complex generative tasks, such as conversation, the measurement of confidence scores is also challenging. For example, traditional confidence estimation methods may not be well-suited for measuring LLMs’ token-by-token generation \cite{gupta2024language}. This difficulty in accurately assessing model confidence complicates the deferral policy learning within cascades, potentially leading to suboptimal routing of queries and affecting overall system efficiency.

\paragraph{Constraints with API-Based LLMs} More practically, many current LLMs, especially those accessible only via APIs, do not support fine-tuning. This restriction hinders our ability to tailor each LLM in the cascade to specific tasks in an online learning setting, presenting a significant limitation for customizing LLM-only cascades.

\paragraph{Computational Costs} The most important hurdle in a more complex cascade is the significant computational expense associated with online learning. For example, in an LLM-only cascade, updating LLMs in real-time (even if using efficient training techniques like LoRA) demands substantial resources, making it impractical for many applications. This computational burden not only affects scalability but also limits the frequency and extent to which LLMs can be updated, impacting the cascade's adaptability and performance. As analyzed in Appendix \ref{app:equi}'s cost equilibrium, when the learnable model space grows too large, the training costs would offset the saved inference costs, which is against our initial motivation of enhancing cost-efficiency.